\documentclass{article}
\usepackage{lineno}
\usepackage{ijcai19}
\usepackage{times}
\usepackage{soul}
\usepackage[hidelinks]{hyperref}
\usepackage[small]{caption}
\usepackage{graphicx}
\usepackage{amsmath}
\usepackage{booktabs}
\usepackage[title]{appendix}
\usepackage{float}
\usepackage{graphicx}
\graphicspath{ {fig/} {../fig/} {../code/} }
\usepackage{subcaption}
\usepackage[short]{optidef}
\usepackage{enumerate}
\usepackage{mathtools}
\usepackage{amsmath}
\usepackage{amsthm}
\usepackage{bm}
\usepackage{subfiles}

\newtheorem{xmpl}{Example}   

\newtheorem{defn}{Definition}
\newtheorem{proposition}{Proposition}
\newtheorem{lemma}{Lemma}

\usepackage{ifthen}
\newboolean{full}
\setboolean{full}{true}
\newboolean{anon}
\setboolean{anon}{false}

\title{The Price of Local Fairness in Multistage Selection\ifthenelse{\boolean{full}}{}{\footnote{The full version of this paper is available at the following link: \url{https://hal.inria.fr/hal-02145071/document}}}}

\author{
\ifthenelse{\boolean{anon}}{
Anonymous Author(s)
}{
Vitalii Emelianov$^1$
\and George Arvanitakis$^1$
\and Nicolas Gast $^{1}$
\and Krishna Gummadi$^2$
\And Patrick Loiseau$^{1,2}$
\affiliations
$^1$Univ. Grenoble Alpes, Inria, CNRS, Grenoble INP, LIG\\
$^2$Max Planck Institute for Software Systems
\emails
\{vitalii.emelianov, george.arvanitakis, nicolas.gast, patrick.loiseau\}@inria.fr,
gummadi@mpi-sws.org
}
}

\hypersetup{draft}

\begin{document}

\maketitle

\begin{abstract}

The rise of algorithmic decision making led to active researches on how to define and guarantee fairness, mostly focusing on one-shot decision making. In several important applications such as hiring, however, decisions are made in multiple stage with additional information at each stage. In such cases, fairness issues remain poorly understood. 

In this paper we study fairness in $k$-stage selection problems where additional features are observed at every stage. 
We first introduce two fairness notions, local (per stage) and global (final stage) fairness, that extend the classical fairness notions to the $k$-stage setting. We propose a simple model based on a probabilistic formulation and show that the locally and globally fair selections that maximize precision can be computed via a linear program. 
We then define the \emph{price of local fairness} to measure the loss of precision induced by local constraints; and investigate theoretically and empirically this quantity. 
In particular, our experiments show that the price of local fairness is generally smaller when the sensitive attribute is observed at the first stage; but globally fair selections are more locally fair when the sensitive attribute is observed at the second stage---hence in both cases it is often possible to have a selection that has a small price of local fairness and is close to locally fair.

\end{abstract}

\section{Introduction}
\label{section: introduction}


The rise of algorithmic decision making in applications ranging from hiring to crime prediction \cite{Policing} has raised critical concerns regarding potential unfairness towards groups with certain traits, supported by recent empirical evidences of discrimination \cite{careerAds,Propublica}. 
This led to a fast-growing body of literature on what fairness in algorithmic decision making is and how to guarantee it (see related works below). 

The existing literature typically considers one-shot decision processes whereby, from a set of features observed about an individual---one of them being a `sensitive feature' based on which discrimination is defined---, one needs to decide whether or not to ``select'' him/her (where select can mean hire, grant a loan or parole, etc. depending on the context). The problem in this setting is how to learn a decision rule from past data that respects certain fairness constraints. 
In many applications, however, decisions are made in multiple stages. In hiring for instance, a subset of candidates is first selected for interview based on resume (or high-level candidate's information) and a final selection is then made from the subset of interviewed candidates. 
In police practices, there are often multiple stages of decisions with increasingly high levels of investigation of the individuals not released at the previous stage; as for instance in the famous stop-question-and-frisk practice by the New-York City Police Department. 

A distinctive specificity of the multistage setting, besides the fact that decisions are made in multiple stages, is that in many cases additional features get known at later stages for the subset of individuals selected at earlier stages, but one needs to make the early-stage selection without observing those features. This raises a number of new questions that are fundamental to fair multistage selection. First, given that there are multiple layers of decisions, \emph{how should fairness be defined?} In particular, should it be defined at each individual stage, on the final decision, or otherwise? Second, given that one has to make decisions with only partial information at early stages, \emph{how to make an optimal selection?} Finally, given that the sensitive feature can be observed at different stages, \emph{is it better to observe the sensitive feature at earlier or later stages (for both fairness and utility)?} This last question intuitively relates to recurrent public debates such as ``should gender identification be removed from CVs?''. 


In this paper, we study the $k$-stage selection problem, in which there is a fixed limit (or budget) of candidates that can be selected at each stage (as is natural in the applications discussed). To tackle the questions above, we propose a simple model based on a probabilistic formulation in which we assume perfect knowledge of the joint distribution of features at all stages and of the conditional probability of being a desirable candidate conditioned on feature values. Based on this model, we are then able to make the following contributions. 
 
We introduce two meaningful notions of fairness for the $k$-stage setting: \emph{local fairness} (the selection is fair at each stage) and \emph{global fairness} (only the final selection needs to be fair). These definitions extend classical group fairness notions for one-stage decision making (such as demographic parity or equal opportunity) to the multistage setting and they apply regardless of when the sensitive feature is observed (at first stage or later). We show that local fairness implies global fairness and we propose a linear formulation of the problem that allows us to compute the selection algorithm that maximizes precision while satisfying (local or global) fairness and per-stage budget constraints in expectation. As local fairness is a more restrictive condition, the precision of the optimal globally fair algorithm is naturally higher than for the locally fair algorithm. To capture this gap, we define the \emph{price of local fairness} ($PoLF$) as the ratio of the two and prove a simple upper bound---showing that imposing local fairness cannot be arbitrarily bad.
We also define the notion of violation of local fairness ($VoLF$) to capture how far from locally fair the optimal globally fair algorithm is.

Finally, we conduct a numerical study in a two-stage setting using three classical datasets. Our results show that the $PoLF$ can be large (up to $1.6$ in some cases). This implies that in some cases, enforcing local fairness constraints can reduce the precision by 60\% compared to a globally fair algorithm. The $VoLF$ is also sometimes large (up to $0.6$ in our experiments), which means that imposing only a global fairness constraint can be highly unfair at intermediate stages. 
We finally compare what happens when the sensitive feature is observed at the first stage or at the second stage. 
We find that the $PoLF$ is generally higher when the sensitive feature is observed at the second stage; while conversely the $VoLF$ is generally higher when the sensitive feature is observed at the first stage. 
These results show that, in most cases, it is possible to get at least approximate fairness at each stage and precision close to globally-fair optimal together; either by imposing local fairness if the sensitive feature is observed at first stage (where $PoLF$ is small) or by hiding the sensitive feature at first stage and using a globally fair algorithm (which is close to locally fair since $VoLF$ is then small). 


Overall, our results provide intuitive answers towards better understanding fairness in multistage selection. To that end, we intentionally used the simplest model that captures the main features of a multistage selection problem and how an optimal selection algorithm is affected by the fairness notion considered and the time at which the sensitive feature is observed---rather than using a more practical but complex model. We believe that it is a good abstraction to start with, but we elaborate further on our model's limitations in Section~\ref{section: conclusion}. 
\ifthenelse{\boolean{full}}
{Some details (proofs, additional formalization and experimental results) are omitted throughout the paper and deferred to the appendices.}
{Due to space constraints, some details (proofs, additional formalization and experimental results) are omitted and can be found in the appendices of the full version.}

\subsubsection{Related Works}
As mentioned earlier, there have been many recent works on defining fairness and constructing algorithms that respect those definitions for the case of one-stage decision making \cite{Pedreshi08a,Dwork11,Kleinberg17a,Hardt:2016,Zafar17a,Chouldechova17a,Corbett-Davies:2017,Kilbertus17,Lipton18a}. Most of those works focus on classification and propose definitions of fairness based on equating some combinations of the classification outcome (true positives, true negatives, etc.).  
In this work, we focus on two classical notions of fairness for the one-shot classification setting: demographic parity (or disparate impact) and equal opportunity (or disparate mistreatment) \cite{Hardt:2016,Zafar17a}. 
There are also works on fairness in sequential learning \cite{Joseph16a,Jabbari17a,Heidari18a,Valera18a}. The model in those papers is to sequentially consider each individual and make decision for them, but there is no notion of refining selection through multiple stages by getting additional features.

Closer to our work, a few papers investigate multistage classification/selection without fairness considerations \cite{Senator05a,Trapeznikov12a}. \cite{Schumann2018a} model the interview decisions in hiring as a multi-armed bandit problem and consider getting extra features at a cost for a subset of candidates, but they do not have fairness constraints: they propose an algorithm for their bandit problem and show that it leads to higher diversity than other algorithms.

To the best of our knowledge, our paper is the first that proposes concrete fairness notions for multistage selection and algorithms to maximize utility under fairness constraints. The only other papers discussing fairness in the context of two-stage or composed decision making are \cite{Bower17a,Dwork19a}, but they do not model additional features becoming available at the second stage for the subselected individuals, which is the key element of our analysis. 

In recent work, \cite{Kleinberg18a} consider the problem of selecting a subset of candidates to interview and show that under some condition, imposing diversity may increase utility when there is implicit bias. Their model, however, assumes no statistical knowledge of the features revealed at second stage, and they only maximize the sum of values of subsected candidates (effectively reducing to one-stage). In contrast, we do not consider implicit bias but we do model the second-stage process. Interestingly, our optimal solution also introduces diversity at the first stage selection, but for different reasons. 




\section{Multistage Selection Framework}
\label{section: problem setup}

\subsection{Basic Setting and Notation}

Assume that there are $n$ candidates,\footnote{We use the term
  candidates in a generic sense to refer to elements of the initial
  set that can be selected.} each described by $d$ features, and
consider the following $k$-stage selection process.  At the first
stage, we observe some of the features $x_1, \dots, x_{d_1}$ of the $n$
candidates where $d_1<d$. We then select $n_1$ of them that ``pass" to
the second stage. At the second stage, we observe some extra features
of these $n_1$ candidates $x_{d_1+1}, \dots, x_{d_2}$ ($d_1 < d_2$)
that were not known at the previous stage. Using the features of both
stages, we do a selection, from the $n_1$ that passed the first stage
of $n_2 \le n_1$ candidates that pass to the next stage, and so on. At
the last stage $k$, we observe all $d_k = d$
features of the $n_{k-1}$ candidates and select $n_k$ among those who
passed the stage $k-1$.


We assume that each candidate is endowed with a \emph{label}
$y \in \{0, 1\}$, which encodes whether the candidate is ``good" or
``bad'' according to the purpose of the selection, i.e., if $y=1$ we
would like to have this candidate in our final selection, if $y=0$ we
would prefer not. The label
$y$ is not known until the end and is therefore not available to make
the selection.

We assume that the
decision maker knows the joint distribution of features and the
conditional probability that expresses the probability that the candidate
is ``good" given all its features.  We will denote by
$p_{x_1 \dots x_d} = P(x_1, \dots, x_d)$ the probability to observe a specific realization
of features and by $ p_{x_1 \dots x_d}^{y=1}=P(y=1 | x_1, \dots, x_d)$ the probability that a
candidate is good ($y=1$) given its features $x_1\dots x_d$.

%

\subsection{Probabilistic Selection and Budget Constraints} 
\label{sec:proba_algo}

In the following, we will consider a class of selection algorithms
that perform a probabilistic selection of candidates. Such an
algorithm takes as an input a list of probability values
$p^{(i|i-1)}_{x_1\dots x_{d_{i}}}$ for all stages $i\in\{1\dots k\}$
and all possible combination of features.  Then, for each candidate
that passed stage $i-1$ and has features $(x_1\dots x_{d_i})$, the
algorithm selects this candidate for the next stage with probability
$p^{(i|i-1)}_{x_1\dots x_{d_{i}}}$, with the convention that everyone
passes stage 0.

For each stage $i$, we define a binary predictor $\hat{y}_i$ that is
equal to $1$ if the candidate is selected at stage $i$ (by convention, $\hat{y}_0=1$ for all candidates).  We assume that, \emph{on average}, the number of
candidates that can be selected by the algorithm at stage $i$ is at
most $\alpha_i n$ and exactly $\alpha_k n$ for the last stage, with
$1\ge\alpha_1\ge\dots\ge\alpha_k$. 
We denote by $\bm{\alpha}_{-k} = (\alpha_1, \cdots, \alpha_{k-1})^T$ the selection sizes of the first $k\!-\!1$ stages. 




\subsection{Performance Metric}


We measure the performance of a given selection algorithm in terms of
precision.  The precision is the fraction of the selected candidates
that indeed were ``good" for selection:\\[-2mm]
\begin{equation*}
  \mathrm{precision} \!=\! \frac{\text{True Positive }}{\text{True
                       Positive} \!+ \! \text{False Positive }}
                     \!=\!P(y\!=\!1| \hat y_k \!=\! 1),
\end{equation*}
where the denominator is the number of selected candidates.

The choice of precision may seem arbitrary but it is in fact a very
natural metric when the size of the final selection is fixed as in our
setting. Indeed, maximizing precision is then equivalent to maximizing
most other meaningful metrics as formalized in the next proposition.
\begin{proposition}
Assume that the selection size  $P(\hat y_k = 1)$ is fixed (to $\alpha_k$). Then maximization of precision is equivalent to maximization of true positive rate, true negative rate,  accuracy and $f_1$-score; and to minimization of false positive rate and false negative rate.
\label{proposition: metric equivalence}
\end{proposition}

Additionally, there are many realistic $k$-stage selection processes for
which precision can be used as a utility metric. 
\begin{xmpl}
A bank decides to whom it will give
entrepreneur loans. The procedure is in two stages: at first, $n$
candidates fill in an application form, and the first $d_1$ features of
each candidate are obtained. Some candidates are then invited for an
interview which brings additional features of those candidate that the
bank can use for its final decision of selecting $n_2$ candidates. If the profit of giving a loan to
a trustworthy candidate is $c_p$ and if a candidate that does not pay
a loan costs $c_l$, then the average gain can be written as:
\begin{small}
\begin{align*}
U_{\text{bank}} &= (c_p + c_l) \cdot n_2 \cdot P(y=1|\hat y_2 =1) - n_2 \cdot c_l.
\end{align*}
\end{small}
\end{xmpl}
In this example, $c_p$, $c_l$ and selection size $n_2$ are
fixed. Hence, maximizing precision or utility is equivalent.


\section{Fairness Notions in Multistage Setting}
\label{section: fairness notions}

In this section, we propose new notions of fairness for the multistage selection problem. 
We assume that there exists, amongst all
features that describe candidates, a sensitive feature $x_s$ that indicates
whether or not a candidate belongs to a sensitive group that should not be discriminated against.

The literature has introduced multiple definitions of fairness for the single-stage setting (and it is worth mentioning that in most of the cases those fairness criteria cannot be satisfied simultaneously
\cite{Chouldechova17a}). The most relevant notions in
the context of selection problems are \emph{Demographic Parity}
(DP) and \emph{Equal Opportunity} (EO). We first recall the definition of these fairness
criteria in the traditional setting of single-stage selection. We then extend them to the multistage setting by showing that there are essentially two relevant notions of fairness: \emph{local} and \emph{global} fairness.

\subsection{Classical Fairness Notions in Single-Stage}


Let $\hat y$ be a binary predictor that decides which candidates belong to the selection. The first fairness definition, widely known as demographic parity, states that the predictor $\hat y$ is fair if it is statistically independent from $x_s$.
\begin{defn}[Demographic Parity, DP]
The binary predictor $\hat y$ satisfies DP with respect to $x_s$ if $\hat y $ and $x_s$ are independent:
\begin{equation}
P(\hat y = 1 | x_s = 0) = P(\hat y = 1| x_s=1) .
\end{equation}
\label{DPdef}
\end{defn}\vspace{-.5cm}

DP does not take into account the actual label $y$. \cite{Hardt:2016,Zafar17a} argue that DP is not the most relevant notion of fairness in cases where we have ground truth on the quality of the candidates (which is our case since we assume statistical knowledge of the probabilities of labels). In such cases, one might want to be fair among the candidates that are worth selecting, a metric called Equal Opportunity \cite{Hardt:2016} (an equivalent notion called disparate mistreatment is proposed in \cite{Zafar17a}): 
\begin{defn}[Equal Opportunity, EO \cite{Hardt:2016}]
The binary predictor $\hat y$ satisfies EO with respect to $x_s$ if $\hat y $ and $x_s$ are independent given that $y=1$:
\begin{equation}
P(\hat y = 1 | y=1, x_s = 0) = P(\hat y = 1| y=1, x_s=1).
\end{equation}
\label{EOdef}
\end{defn}\vspace{-.4cm}

In the remainder, we systematically consider DP and EO.

\subsection{Local and Global Fairness in Multistage}

Existing fairness notions apply to single-stage selection, where we have
only one binary predictor $\hat y$. In the case of $k$-stage
selection, we have $k$ binary predictors
$\hat y = (\hat y_1, \dots, \hat y_k)$. In this section, we develop
different notions of fairness that extend existing notions
to the $k$-stage selection setting.

We propose three definitions that we believe correspond to three reasonable notions of fairness. The high-level idea of each definition is depicted on Figure \ref{fig:fairness_definitions}. For the sake of brevity of exposition, we present the formal definitions for the demographic parity criterion, the translation to EO (or to any other fairness notion) being straightforward. 

\begin{figure}
\centering
\includegraphics[width=0.7\linewidth]{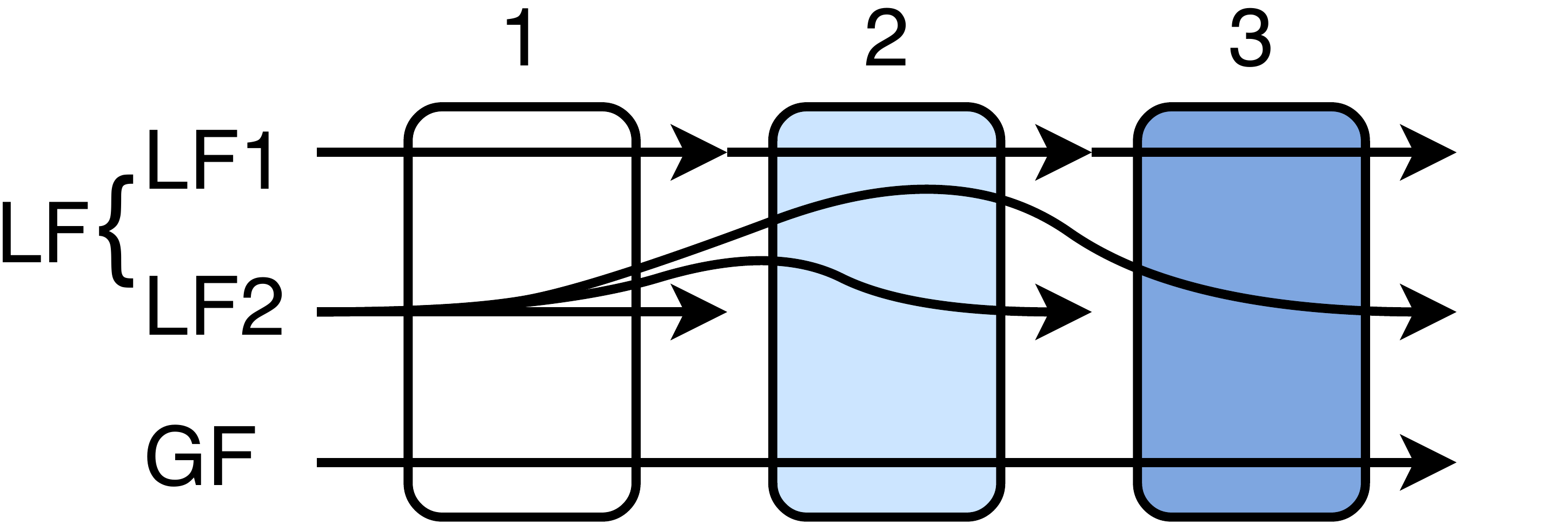}
\caption{Illustration of the different fairness definitions.}
\label{fig:fairness_definitions}
\end{figure}

The first fairness notion, local fairness 1 (LF1), imposes that the selection be fair at every stage with respect to the set of candidates that reached that stage. In other words the selection of each stage $i$ is fair with respect to the population that ``passed" stage $i-1$.
\begin{defn}[Local Fairness 1, LF1]
  A $k$-stage selection algorithm satisfies LF1 if (for the case of
  DP), $\forall i \in \{ 1, \cdots, k\}$:\\[-3mm]
\begin{equation*}
  P(\hat y_i\!=\!1 | \hat y_{i-1} = 1, x_s\!=\!0) = P(\hat y_i\!=\!1 | \hat y_{i-1}\!=\!1, x_s\!=\!1).
\end{equation*}
\end{defn} 

The second fairness notion that we propose, local fairness 2 (LF2),
prescribes that the selection should be fair at each stage with
respect to the initial set of candidates.  
\begin{defn}[Local Fairness 2, LF2]
A $k$-stage selection algorithm satisfies LF2 if (for the case of DP),  $\forall i \in \{ 1, \cdots, k\}$:\\[-3mm]
$$P(\hat y_i=1 | x_s=0) = P(\hat y_i=1 | x_s=1).$$
\end{defn}

In the last definition, global fairness (GF), we allow the predictor $\hat y_i$ to be unfair at each stage before the last, but we require the final decision $\hat y_k$ to be fair with respect to the initial set of candidates.
\begin{defn}[Global Fairness, GF]
A $k$-stage selection algorithm satisfies GF if (for the case of DP):\\[-3mm]
\[
P(\hat y_k = 1| x_s=0) = P(\hat y_k =1 | x_s=1) .
\]
\end{defn}
Note that the above definitions can be adapted to EO by conditioning on $y=1$ in all formulas.

\subsection{Equivalence between LF1 and LF2}

In the following proposition, we show that both notions of local
fairness, LF1 and LF2 are equivalent. Therefore in the rest of the
paper, we will simply name a multistage selection algorithm that
satisfies LF1 (and thus LF2) as a being \emph{locally fair} (LF). An
algorithm satisfying the global fairness definition will be called
\emph{globally fair} (GF).
\begin{proposition}[Relations between fairness notions]
  \label{proposition: fairness notions relation}
  For both DP and EO:
  \begin{enumerate}
  \item A selection algorithm satisfies LF1 if and only if it
    satisfies LF2. We call such an algorithm locally fair (LF). 
  \item A locally fair selection algorithm is globally fair (GF).
  \end{enumerate}
  \label{proposition: equivalence}
\end{proposition}


\section{Utility Maximization as a Linear Program}
\label{section: linear utility maximization}


Our goal is to find the binary predictors
$(\hat y_1, \dots, \hat y_k)$ corresponding to stages from 1 to $k$,
respectively, that maximize precision while respecting budget and
fairness constraints:
\begin{align}
  \begin{array}{cll}
    \displaystyle\max_{\hat y_1, \dots, \hat y_k}
    &P(y=1|\hat y_k = 1)\\
    &P(\hat y_i=1)\leq \alpha_i,&i \leq k-1\\
    &P(\hat y_k=1)=\alpha_k\\
    &f_j(\hat y_1, \dots, \hat y_k)= 0,&j \leq t 
  \end{array}
  \label{pb: non-fair1}
\end{align}
where functions $f_j(\cdot)$ of the binary predictors correspond to the fairness constraints we impose. For instance, for a globally fair algorithm (DP) we have only one fairness constraint: $f(\hat y_1, \dots, \hat y_k) = P(\hat y_k = 1| x_s=0) - P(\hat y_k =1 | x_s=1) .$

Using the assumption that the final stage size constraint is $P(\hat y_k = 1) = \alpha_k$ we can
write the precision as follows:\\[-4mm]
\begin{equation}
  P(y\!=\!1 | \hat y_k \!=\! 1) \!=\! \frac{1}{\alpha_k} \! \sum \limits_{x_1 \dots x_d
  } \!\!\! p_{x_1 \dots x_d}^{y=1} p_{x_1 \dots x_d} \! \prod_{j=1}^k \! p_{x_1 \dots
    x_{d_j}}^{(j|j-1)}.
\label{eqppv}
\end{equation}
Using the notation introduced in Section \ref{sec:proba_algo}, the
probability $P(\hat{y}_{i} = 1)$ that candidate passes stage $i$ is\\[-4mm]
\begin{equation}
  P(\hat y_i = 1) = \sum_{x_1 \dots x_{d}} p_{x_1 \dots x_{d}}
  \prod_{j=1}^i p_{x_1 \dots x_{d_j}}^{(j|j-1)}. 
  \label{eq1selection} 
\end{equation}
Hence, the constraints on the selection size 
$P(\hat y_i = 1) \leq \alpha_i$ for $i<k$ and
$P(\hat y_k = 1)  =  \alpha_k$ can be expressed using \eqref{eq1selection}.

The fairness constraints can be developed in the same manner, e.g., for the globally fair case (DP):
$$f(\hat y_1, \dots, \hat y_k) = P(\hat y_k = 1 | x_s=0) - P(\hat y_k = 1 | x_s=1),$$
where $\forall a \in \{0, 1\}$,\\[-5mm]
\begin{equation}
P(\hat y_k \!=\! 1 | x_s\!=\!a)  \!=\! \frac {\sum \limits_{x_i, i \not = s}\!\!\prod_{j=1}^k p_{x_1 \dots x_{d_j}}^{(j|j-1)} \!\! \cdot \! p_{x_1 \dots  x_s=a \dots x_d}}{\sum \limits_{x_i, i\not = s} p_{x_1 \dots x_s=a \dots x_d}}.\!%
\label{gf}
\end{equation}

From \eqref{eqppv}, we see that the objective is not linear in the
variables $p^{(j|j-1)}_{x_1 \dots x_{d_j}}$ due to the product of
probabilities. Similarly, we observe from \eqref{eq1selection} and \eqref{gf} that
the constraints are also not linear in these variables. However, we can show that by using the change of variables
$\tilde p^{(i|i-1)}_{x_1 \dots x_{d_i}}=\prod_{j=1}^i p_{x_1 \dots
  x_{d_j}}^{(j|j-1)}$, it can be made
linear. 
This shows that it is possible to compute the variables
$p^{(j|j-1)}_{x_1\dots x_{d_j}}$ that maximize precision \eqref{pb: non-fair1} using a
linear program (LP) (see details in Appendix A\ifthenelse{\boolean{full}}{}{ of the full version}), which is key to applicability.  It should be noted, however, that the number of variables in (LP) grows exponentially with the number of features.

To distinguish between the different notions of fairness, we will denote by $U^*_{LF}( \bm{\alpha}_{-k}, \alpha_k)$
and $U^*_{GF}( \bm{\alpha}_{-k}, \alpha_k)$ the value of the problem
(LP)---i.e., the maximum utility---when the fairness constraints correspond to local and global
fairness, respectively. Similarly, we will denote by
$U^*_{un}( \bm{\alpha}_{-k}, \alpha_k)$ the optimal precision value when no
fairness constraint are imposed (we call it the \emph{un}fair
case).




\subsection{Solution Properties wrt Budget Constraints}

The selection sizes may be related to some budget or to some physical
resources of our problem and are crucial parameters. As
we show in the next proposition, the optimal utility values are monotonic and concave as functions of
budget sizes $\alpha_1, \dots, \alpha_{k-1}$. This property can be useful for budget optimization  and is illustrated as well on \autoref{fig: optimal u}. 
\begin{proposition}[Monotonicity and concavity]
\label{proposition: monotonicity and concavity}
For $U^*\in\{U^*_{LF},U^*_{GF},U^*_{un}\}$ and any fairness constraints that can be expressed as linear homogeneous equations\footnote{See details in Lemma 1 in Appendix A\ifthenelse{\boolean{full}}{}{ of the full version}.} (such as DP and EO), we have that $U^*(\bm{\alpha}_{-k}, \alpha_k)$ is
\begin{enumerate}
\item  non-decreasing and concave with respect to $ \bm{\alpha}_{-k}$;
\item non-increasing with respect
  to $\alpha_k$.
\end{enumerate}
\end{proposition}
Note that $U^*$ can be concave or convex or none of the two with
respect to $\alpha_k$, depending on the problem's parameters.

\subsection{The Price of Local Fairness}


We are now ready to define our central notion---the \emph{price of
  local fairness}---that represents the price to pay for being fair
at intermediate stages compared to a globally fair solution.
\begin{defn}[Price of Local Fairness, $PoLF$]
Let
$$\text{PoLF}(\bm{\alpha}_{-k}, \alpha_k) = \frac{U_{GF}^*(\bm{\alpha}_{-k}, \alpha_k)}{U_{LF}^*(\bm{\alpha}_{-k}, \alpha_k)}.$$
\end{defn}
It should be clear that the locally fair algorithm is more constrained
than the globally fair. Thus, we have:\\[-4mm] 
\begin{align*}
  U^*_{LF}(\bm{\alpha}_{-k}, \alpha_k) \leq U^*_{GF}(\bm{\alpha}_{-k}, \alpha_k)  \leq U^*_{un}(\bm{\alpha}_{-k}, \alpha_k).
\end{align*}
This implies that the values of $PoLF(\bm{\alpha}_{-k}, \alpha_k)$ are always larger than or equal to 1.
Using only the final selection size $\alpha_k$, it is also possible to
compute an upper bound as follows.
\begin{proposition}[$PoLF$ bound]
\label{proposition: polf bounds}
For all $(\bm{\alpha}_{-k}, \alpha_k)$, we have: \\[-3mm]
$$1 \leq \text{PoLF} (\bm{\alpha}_{-k}, \alpha_k)  \leq \min\left(\frac{1}{\alpha_k}, \frac{1}{P(y=1)}\right).$$
\end{proposition}
For instance, if the final stage selection size is $\alpha_k=0.3$ (as in
our numerical examples), the globally fair algorithm can outperform
the locally fair one by a factor at most 3.33. While this bound is
probably loose, we will see in our numerical example that the $PoLF$ can
be as large as $1.6$ on real data.


\section{Empirical Analysis}
\label{section: empirical analysis}

In this section we implement\footnote{All codes are available at \url{https://github.com/vitaly-emelianov/multistage_fairness/}} the optimization algorithms in order to capture tendencies on real datasets and to provide general insights. We consider the two-stage selection process, since it is the most easily interpretable. Thus, $\bm{\alpha}_{-k} = \alpha_1$ and $\alpha_k = \alpha_2$. In our experiments we use three datasets: \texttt{Adult} \cite{uci}, \texttt{COMPAS} \cite{Propublica} and \texttt{German Credit Data} \cite{uci}. 
We adapt these datasets to our two stage fair selection problem by leaving 6 features, binarizing  them (see details in Appendix D\ifthenelse{\boolean{full}}{}{ of the full version}) and artificially separating  in two stages. We estimate the statistics $p_{x_1 \dots x_d}$ and
$p^{y=1}_{x_1 \dots x_d}$ from data. We then use a linear solver
for the linear program (LP) that gives us the
optimal utility $U^*(\alpha_1, \alpha_2)$ for the fair and 
unfair cases. 

\subsection{Analysis of the Price of Local Fairness}
\label{subsection: influence polf}

We consider three different scenarios: i) the sensitive attribute $x_s$
is observed at the first stage; ii) at the second stage; iii) never used
in the selection process. We distinguish these three cases since it
could happen that the use of the sensitive attribute $x_s$ in decision
making is forbidden at some stages or even at all (by law or other
conventions). Our aim is to compare how the price of local fairness 
behaves in every case.

Let us start with a simple example. We leave 5 features from the
\texttt{Adult} dataset: \textit{sex}, \textit{age},
 \textit{education}, \textit{relationship} and \textit{native country} and consider the
attribute \textit{sex} as sensitive. \autoref{fig:
  optimal u} then shows the values of
$U_{\{un,\; GF,\; LF\}}^*(\alpha_1, \alpha_2)$ as a function of
$\alpha_1$ for fixed $\alpha_2=0.3$ when using the features displayed on top of each subfigure at first stage and the rest at second stage.  
We make two important
observations from this figure.  \emph{First}, the value of $PoLF$ can be significant. From \autoref{fig: optimal u}-(right), we see that for
$\alpha_1 \approx 0.33$, the value of $PoLF$ is about $1.3$, meaning that
the globally fair algorithm achieves $30\%$ larger value of precision
than the locally fair.  \emph{Second}, the gap between LF and GF
algorithms is significantly larger when the sensitive attribute $x_s$
is observed at the second stage.

\begin{figure}
\begin{tabular}{c@{\hskip 0in}c}
{\small (sex, age, education)} & {\small (age, education)}\\
{\small sex (sensitive) at first stage}& \hspace{0.5cm} {\small sex (sensitive) at second stage}\\
\includegraphics[width=0.45\linewidth]{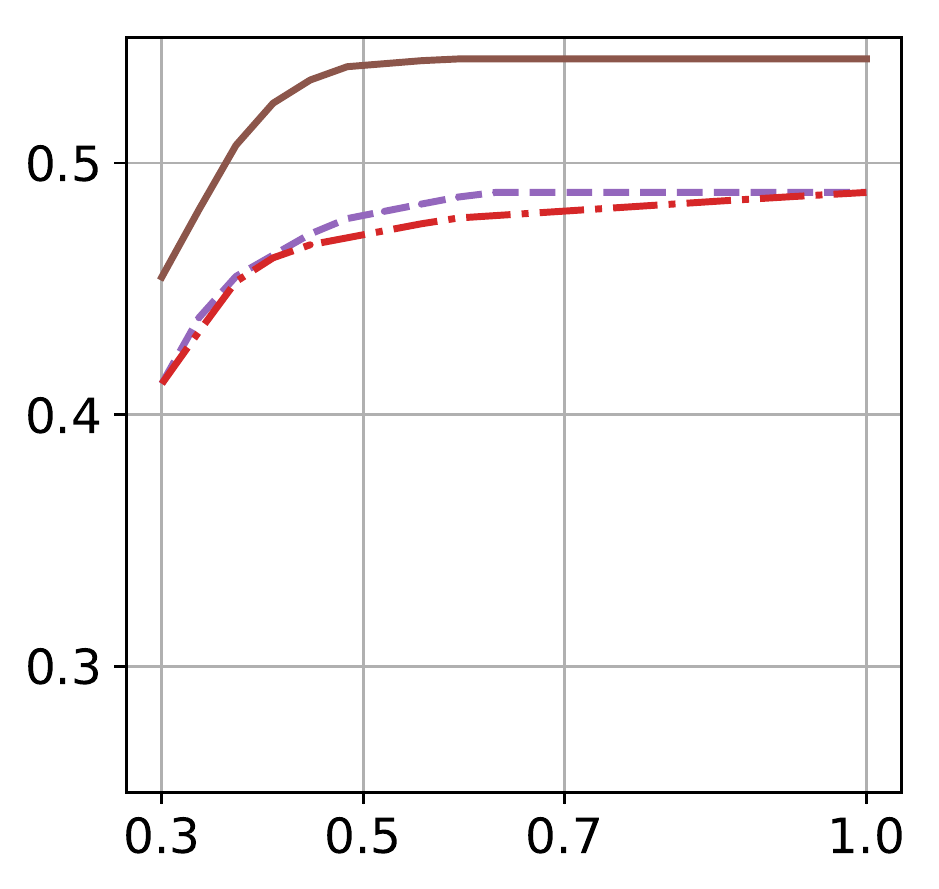} & \includegraphics[width=0.45\linewidth]{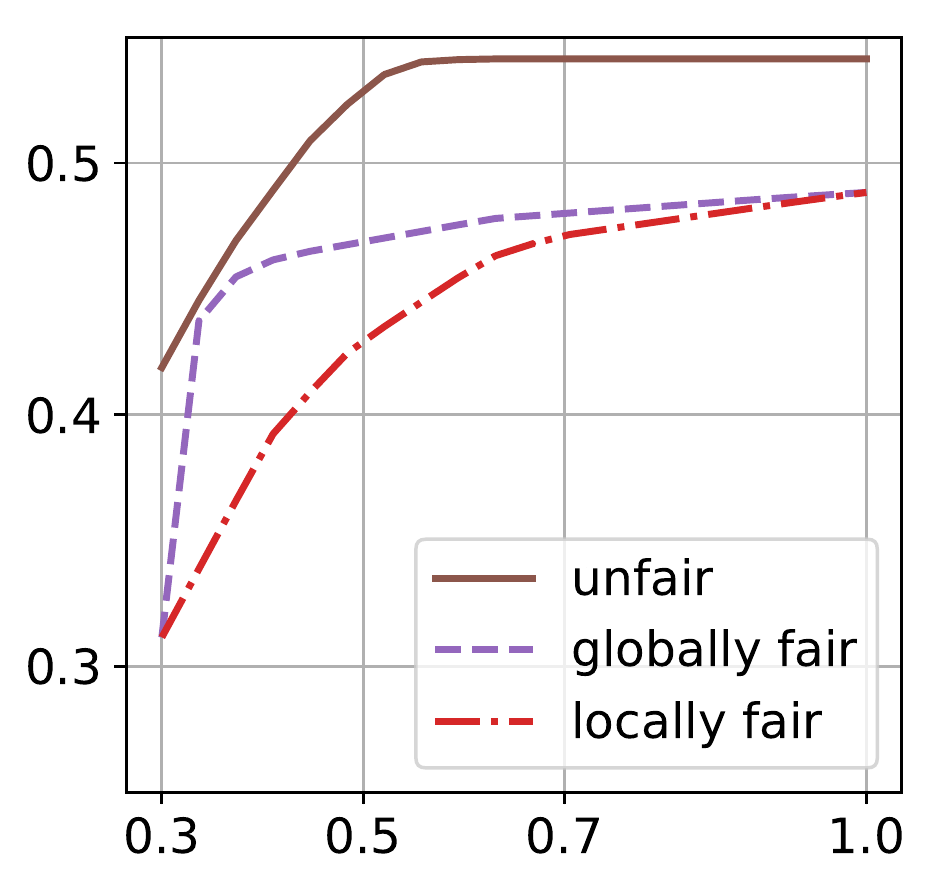} \\[-3mm]
$\alpha_1$ &  $\alpha_1$\\[-3mm]
\end{tabular}
\caption{Utility $U^*(\alpha_1 , \alpha_2=0.3)$ for \texttt{Adult}  dataset (DP).}\label{fig: optimal u}\vspace{-5mm}
\end{figure}

To show that this behavior is significant we calculate the values of
$U_{\{un, \; GF,\; LF\}}^*(\alpha_1, \alpha_2)$ for every possible combination $X = \{x_1, \dots, x_5\}$ of 5 features out of 6 as
decision variables ($x_1, x_2$ at first stage and $x_3, x_4$ at second stage), with one sensitive attribute $x_s = x_5$ that can be
observed at the first stage or at the second stage or not observed at all, and for every possible (discretized) value of $\alpha_1\ge \alpha_2$.  Due to space
constraints we present our results only for the DP definition of fairness; we emphasize that the observations are
robust among the three datasets and the two fairness notions (DP and EO) (see Appendix C for additional results).  
\autoref{fig: polf-all} shows the empirical cumulative distribution
functions $\hat F_{PoLF}(x)$ of the values of $PoLF$ obtained.
We observe that the \emph{price of local fairness
  is significantly lower when the sensitive attribute $x_s$ is
  revealed at the first stage} compared to the case where it is revealed later. This is consistent with
the observation made on
\autoref{fig: optimal u}. 
A possible interpretation is that the LF algorithm has to make a conservative decision at the first stage
and therefore cannot perform well compared to the GF algorithm that is able to compensate (when the sensitive feature $x_s$ is observed) for the unfair decisions that have been made at the first stage.  
It is worth mentioning that we have the same observation for a three-stage algorithm: the later we reveal the sensitive attribute, the higher the values of $PoLF$ we obtain (see Appendix C.4).

\begin{figure}
\begin{center}
\begin{tabular}{c@{\hskip 0in}c@{\hskip 0in}c}
Adult & COMPAS & German\\
\includegraphics[width=0.33\linewidth]{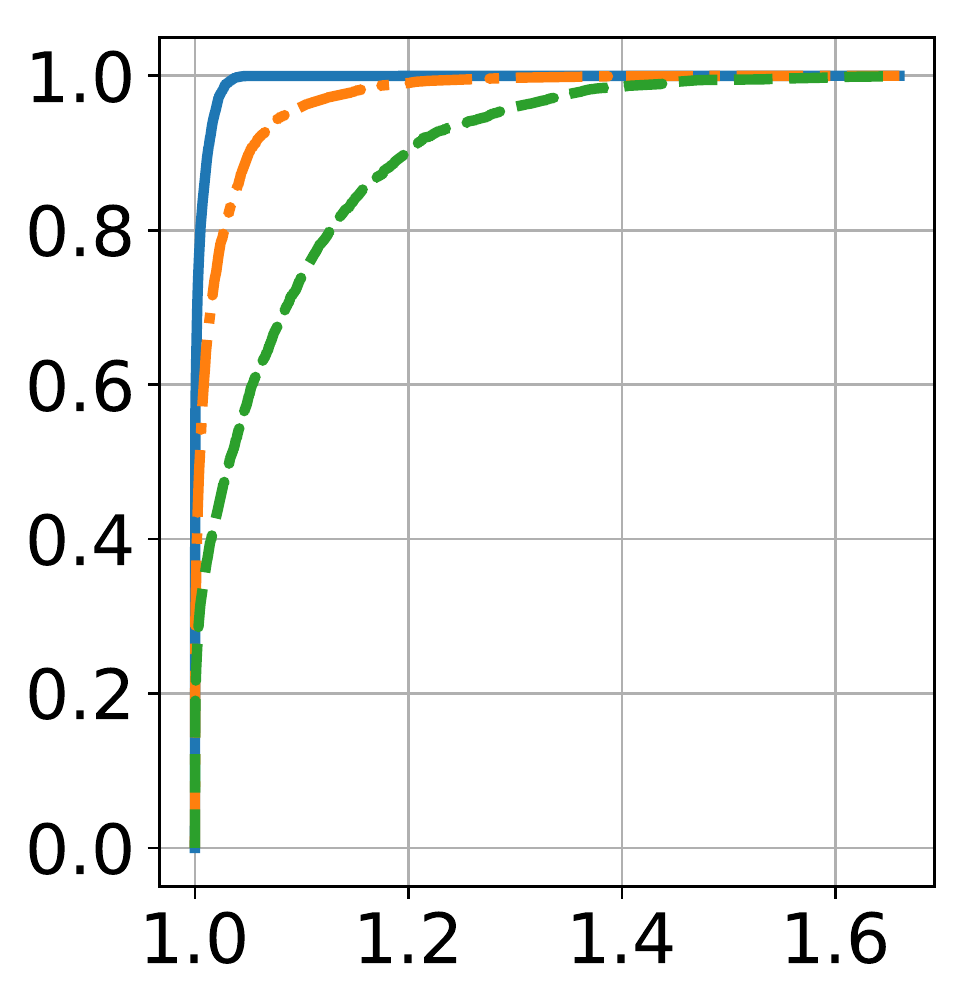} & 
\includegraphics[width=0.33\linewidth]{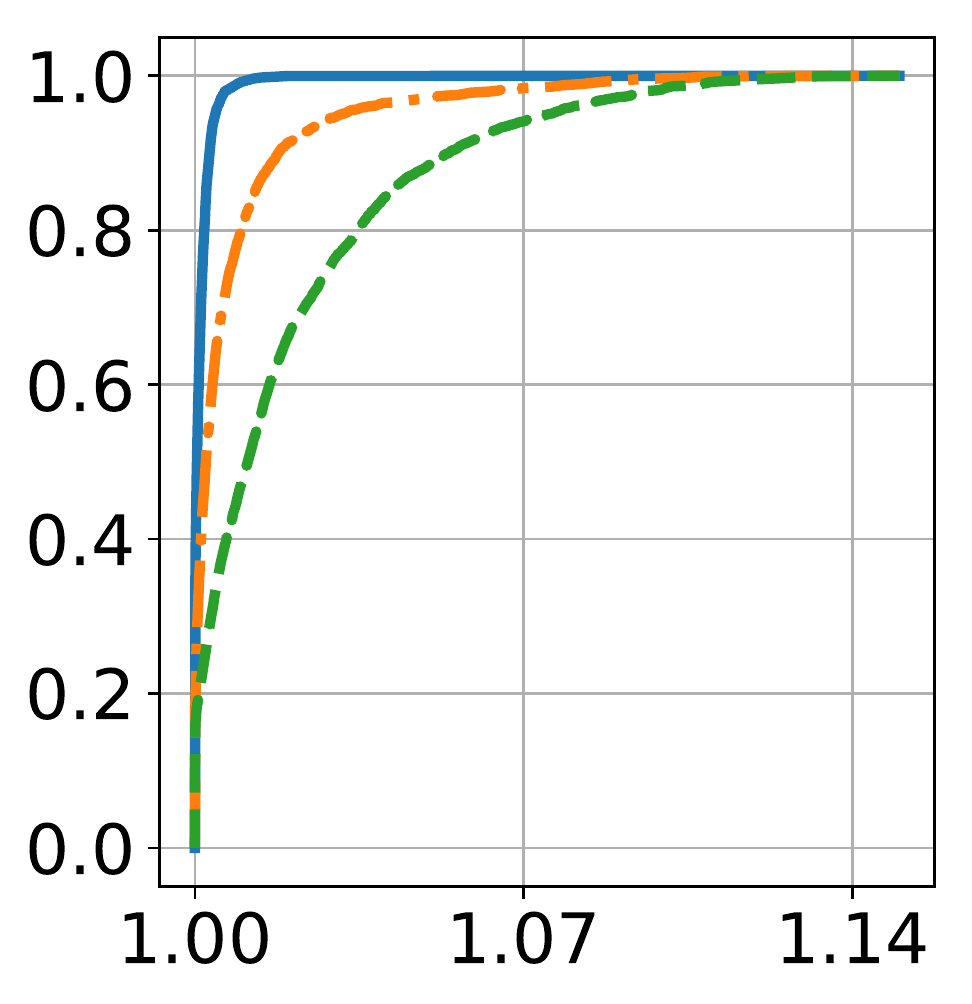} &
\includegraphics[width=0.33\linewidth]{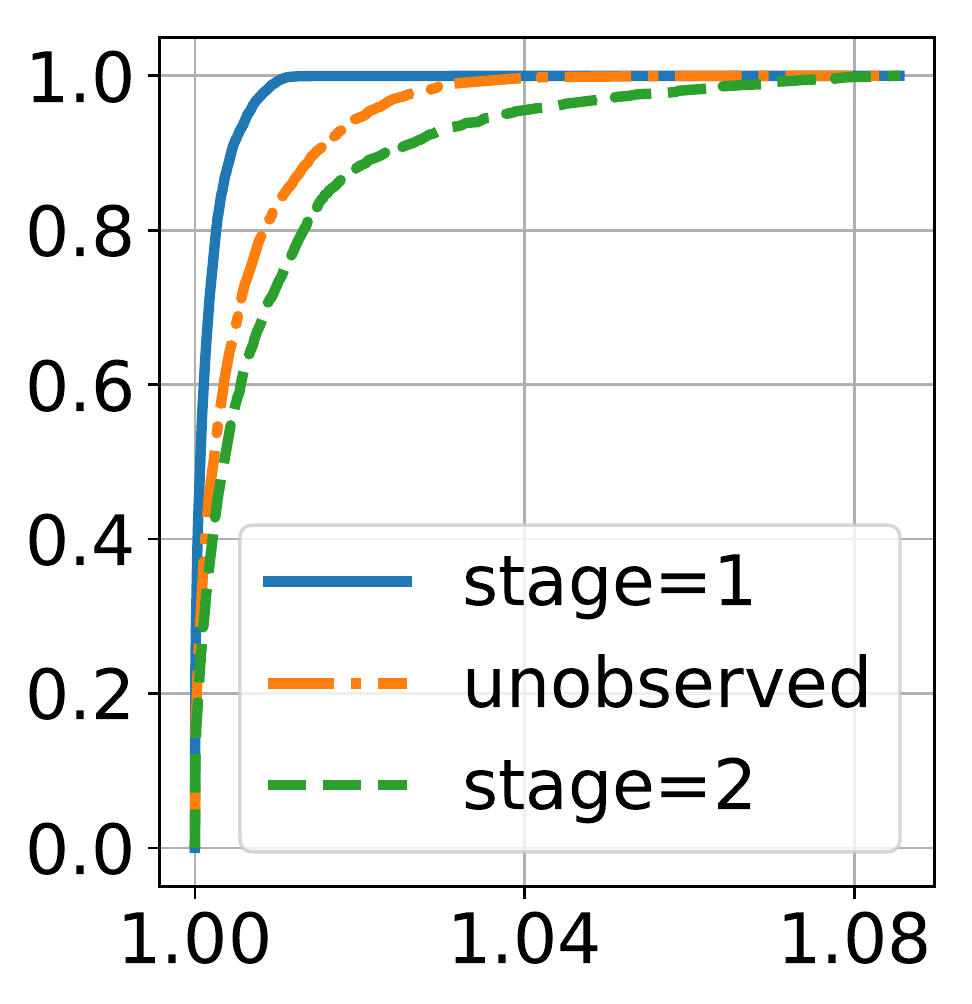}\\[-6mm]
\end{tabular}
\end{center}
\caption{Empirical CDFs of $PoLF$ for all datasets  (DP, $\alpha_2=0.3$).\vspace{-2mm}}
\label{fig: polf-all}
\end{figure}

\subsection{Violation of Local Fairness}

By definition, a globally fair algorithm can violate fairness
constraints at intermediate stages. For a given
budget constraints $\alpha_1,\alpha_2$, we define the violation of local
fairness ($VoLF$) as the absolute value of the fairness constraint violation at the first stage for
the optimal globally fair algorithm. For instance, for DP, this quantity equals:
\begin{align*}
  VoLF(\alpha_1,\alpha_2) = \left|  P(\hat y_1=1 | x_s=0) - P(\hat
  y_1=1 | x_s=1)\right|.
\end{align*}

\autoref{fig: cdf volf-all} shows the empirical cumulative
distribution function of violation of fairness $\hat F_{VoLF}(x)$ for
every value of $\alpha_1\in[\alpha_2;1]$ and for every feature
combination. 
We observe that \emph{the later the sensitive feature $x_s$ is revealed (or
  even not revealed), the more fair at intermediate stages the
  globally fair algorithm is}. One possible explanation is that
an algorithm that cannot observe the sensitive feature $x_s$ at the first stage has
to be more ``cautious'' at every stage to be able to satisfy global
fairness since the exact value of sensitive attribute $x_s$ is not
available. This observation is again robust among different datasets 
and notions of fairness.

\begin{figure}
\begin{center}
\begin{tabular}{c@{\hskip 0in}c@{\hskip 0in}c}
Adult & COMPAS & German\\
\includegraphics[width=0.33\linewidth]{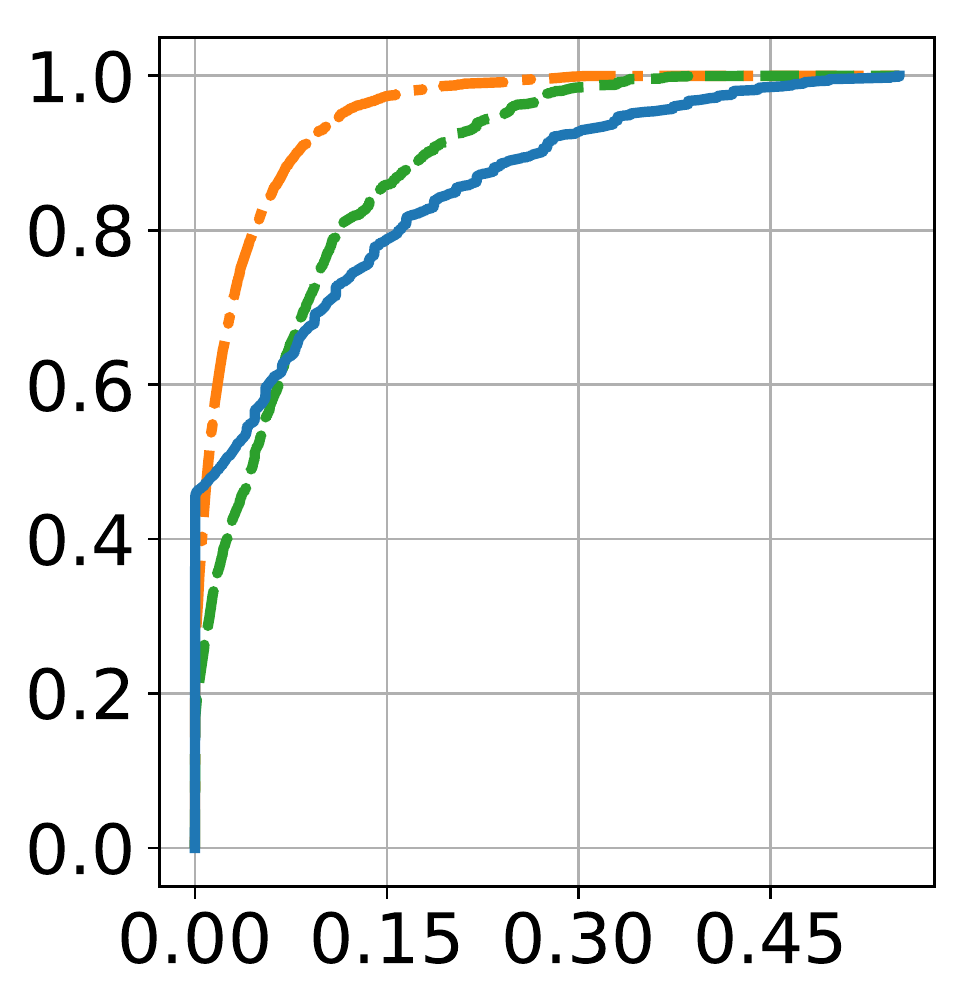} & 
\includegraphics[width=0.33\linewidth]{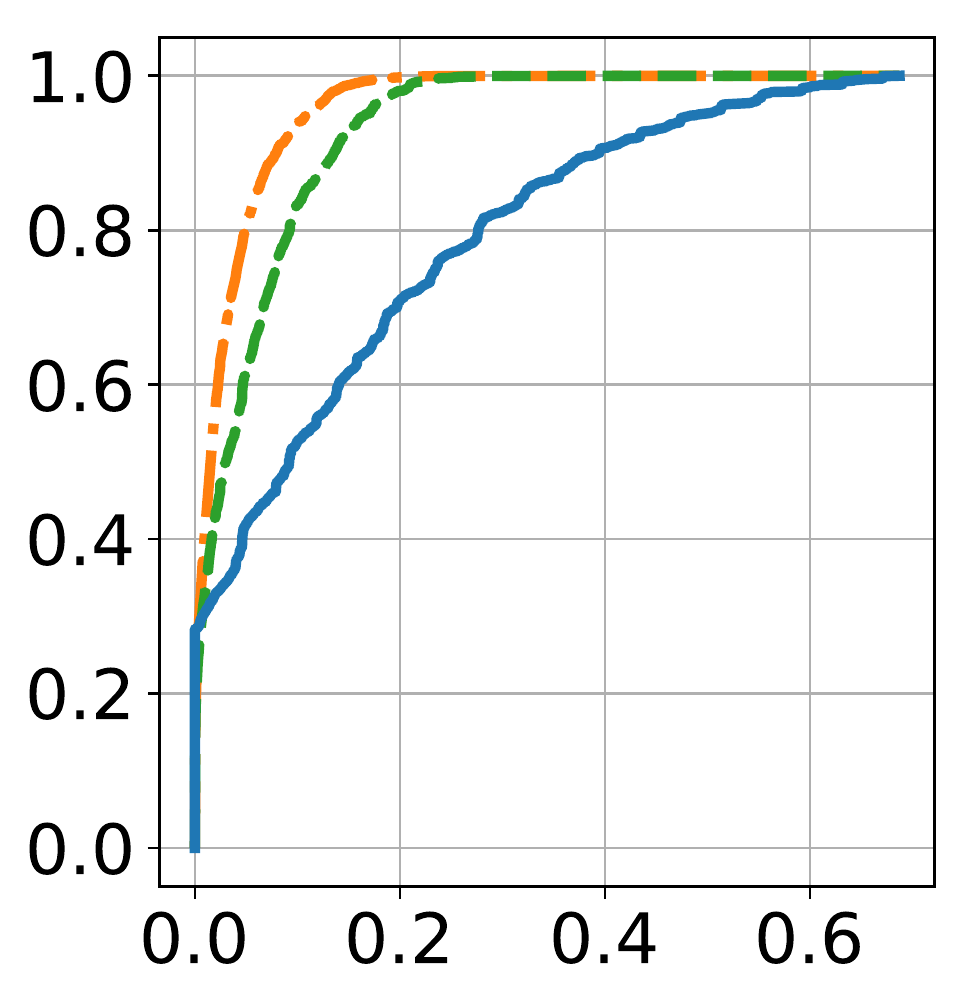} &
\includegraphics[width=0.33\linewidth]{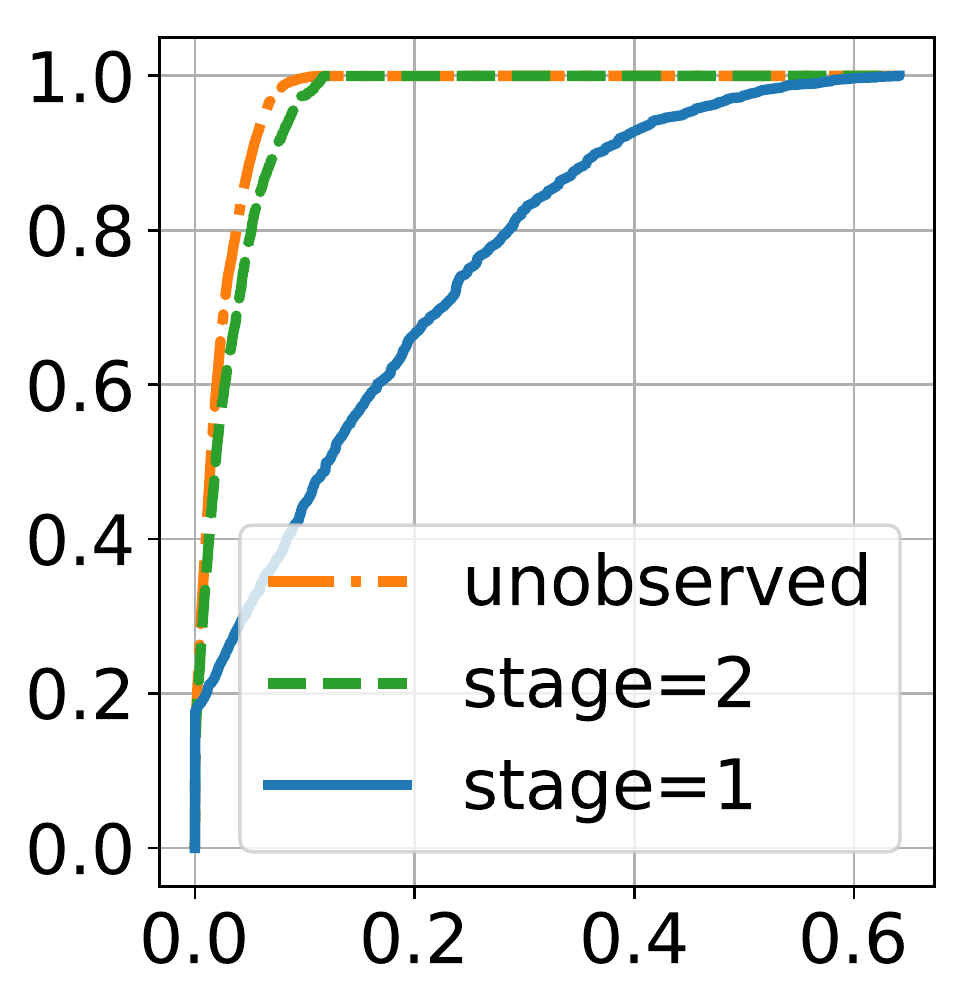}\\[-6mm]
\end{tabular}
\end{center}
\caption{Empirical CDFs of $VoLF$ for all datasets  (DP, $\alpha_2=0.3$).\vspace{-3mm}}
\label{fig: cdf volf-all}
\end{figure}

Finally, on \autoref{fig: polf-volf} we represent the joint distribution of
$PoLF$ and $VoLF$.  As mentioned before, the globally fair algorithm is more
unfair at the intermediate stages when the sensitive feature $x_s$ is
observed from the beginning (left panel), however the price of local
fairness we pay in this case is the smallest one. When the sensitive
feature $x_s$ is observed at the second stage (middle panel) the
globally fair algorithm is more locally fair compared to the previous
case, but the value of $PoLF$ is way larger. Finally, when $x_s$ is
never observed (right panel) the globally fair algorithm is the ``most
locally fair'' among all three settings. 
We finally observe that, while most points have either $PoLF$ small (i.e., using a LF algorithm does not lose much) or $VoLF$ small (i.e., the GF algorithm is almost locally fair), there exist some points---when the sensitive feature is observed at the second stage---where both $PoLF$ and $VoLF$ are large; i.e., imposing local fairness even approximately comes at a significant cost. 

\begin{figure}
\begin{center}
\begin{tabular}{c@{\hskip 0in}c@{\hskip 0in}c}
$x_s$ at first stage & $x_s$ at second stage & $x_s$ unobserved\\
\includegraphics[width=0.33\linewidth]{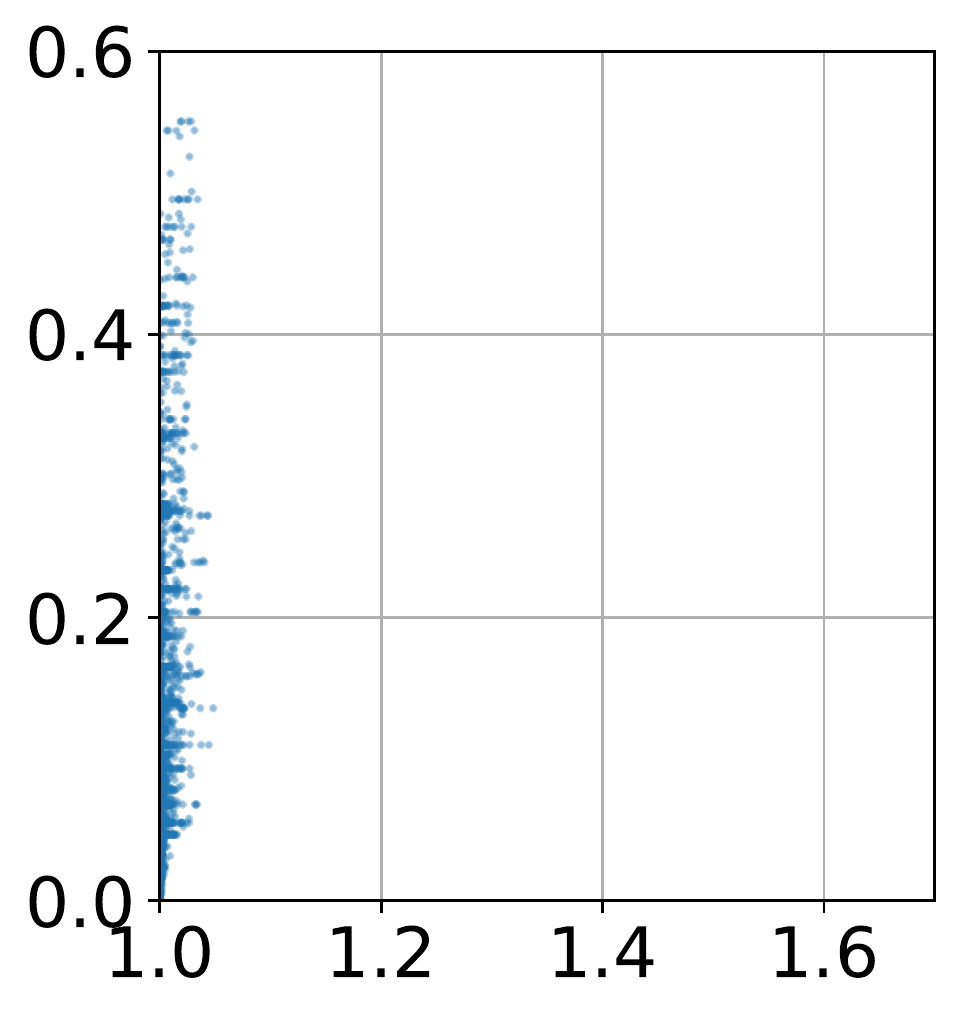} & 
\includegraphics[width=0.33\linewidth]{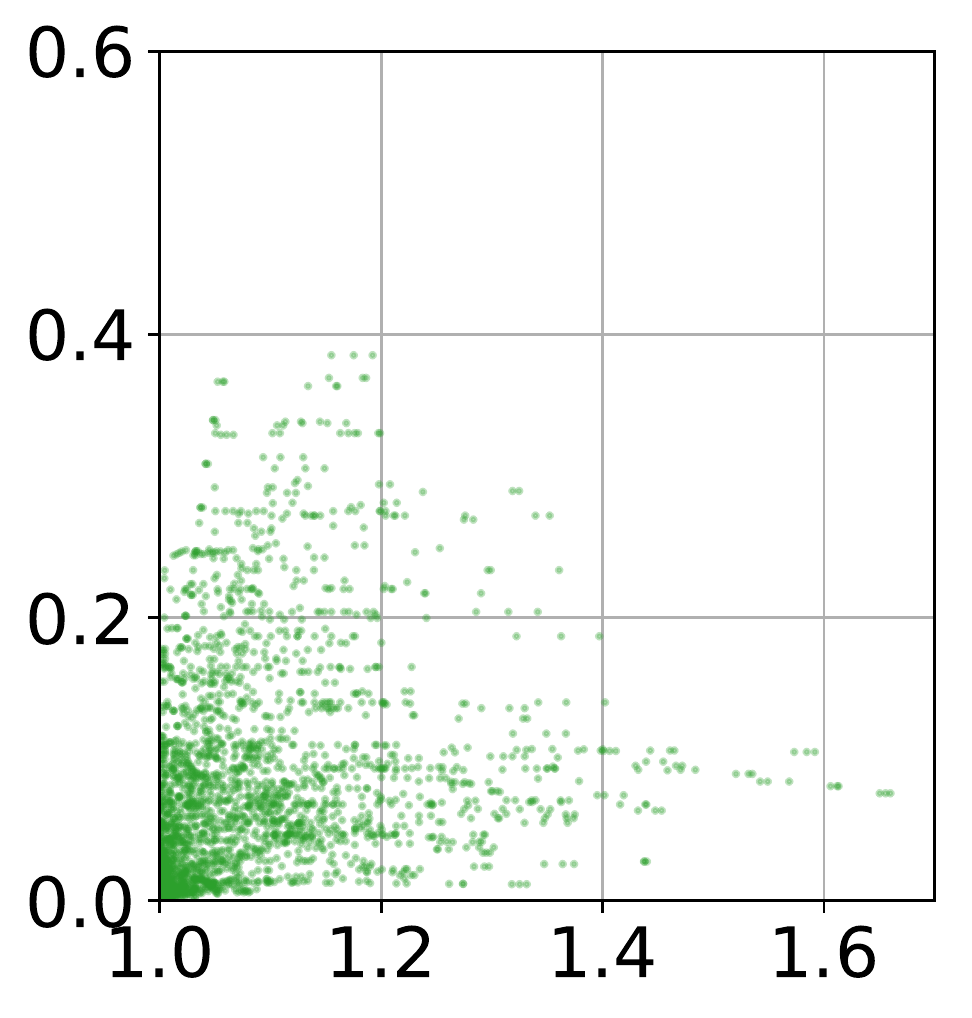} &
\includegraphics[width=0.33\linewidth]{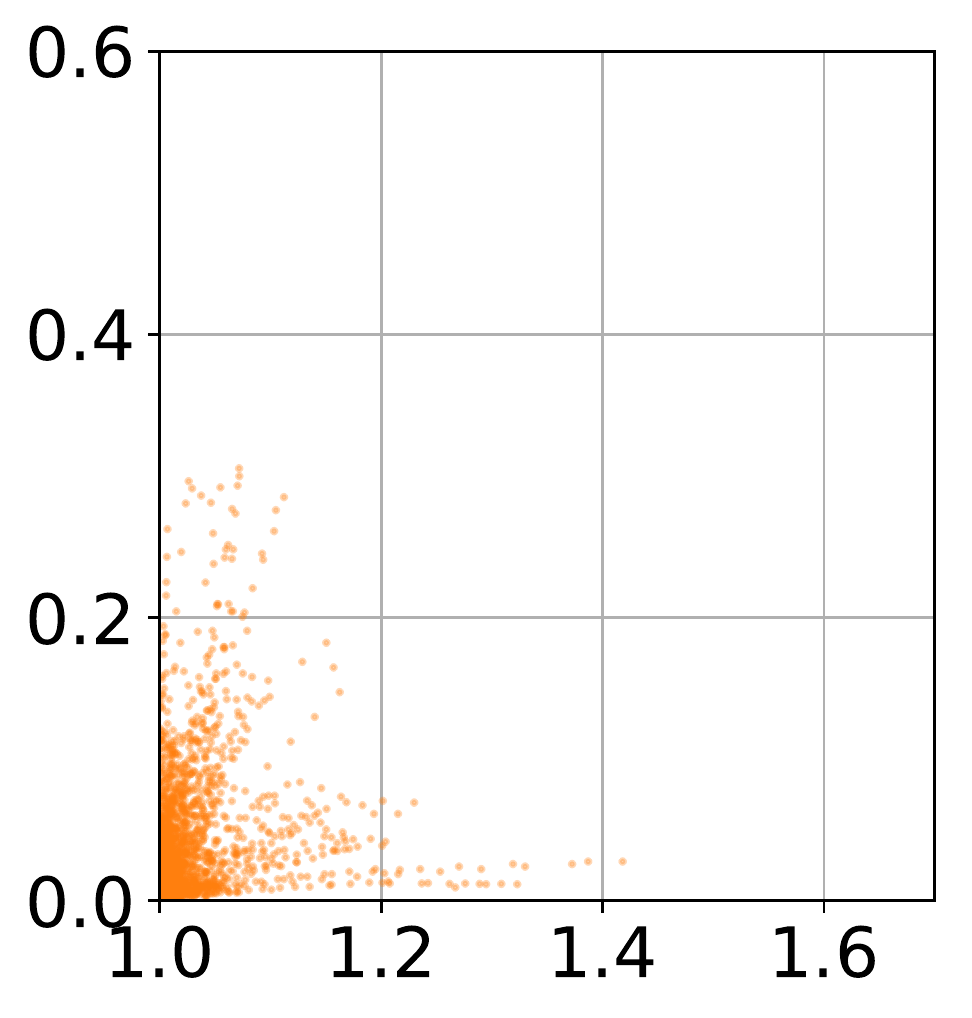}\\[-6mm]
\end{tabular}
\end{center}
\caption{$VoLF$ ($y$-axis) vs $PoLF$ ($x$-axis) for \texttt{Adult}  dataset (DP, $\alpha_2=0.3$).}\vspace{-3mm}
\label{fig: polf-volf}
\end{figure}


\section{Conclusion}
\label{section: conclusion}

In this work we tackle the problem of multistage selection and the
fairness issues it entails. We propose a stylized model based on a probabilistic formulation of
the $k$-stage selection problem with constraints on the
number of selected individuals at each stage that should hold in expectation. We introduce two
different notions of fairness for the multistage setting: local
(under two equivalent variants) and global fairness. Thanks to this
framework, we show that maximizing precision under budget and fairness
constraints can be done via linear programming, which enables for
efficient computation as well as theoretical investigation. In
particular, we analyze theoretically and empirically how the utility
of locally and globally fair algorithms vary with selection budgets,
and we find that globally fair algorithms can lead to
non-negligible performance increases compared to locally fair ones.

One of the main findings of our work is that the stage
at which the sensitive attribute is revealed greatly affects the
difference between the performance of locally and globally fair
algorithms: hiding the sensitive feature at early stages tends to make
globally fair algorithm more fair at intermediate stages. 
While locally fair algorithms may be desirable, our results show that
local fairness does not come for free. They also show that if a
decision maker would like to encourage locally fair selection algorithms, 
there are essentially two choices: either hide the sensitive
feature at the first stage or impose by rules the first stage to be
fair.



Our model allows us to provide elegant insights into the fairness questions related to multistage selection, yet it does a number of simplifying assumptions that naturally restrict its direct applicability. 
\emph{First}, our model ignores the issue that the selection probability at a stage depends on which candidates got selected at the previous stages; i.e., it implicitly makes the approximation that at each stage the number of candidates selected for each feature combination is equal to its expectation. In Appendix E \ifthenelse{\boolean{full}}{}{ of the full version}, we show that this approximation becomes exact as $n$ tends to infinity. \emph{Second}, we assume perfect statistical knowledge of the joint distribution of features and label values, without bias. \emph{Third}, we consider only discrete features and use a non-compact representation of the selection probabilities---this allows us to solve the exact selection problem by using an LP formulation. Relaxing these assumptions, in particular using a more compact representation of the selection algorithm (at the cost of a loss of precision) is an interesting direction of future work. 

\section*{Acknowledgments}

This work was supported in part by the French National Research Agency (ANR) through the ``Investissements d’avenir'' program (ANR-15-IDEX-02) and through grant ANR-16-TERC0012; by the Alexander von Humboldt Foundation; and by a European Research Council (ERC) Advanced Grant for the project ``Foundations for Fair Social Computing'' funded under the European Union's Horizon 2020 Framework Programme (grant agreement no. 789373). The authors also thank Roland Hildebrand for helpful technical suggestions.


\footnotesize
\bibliographystyle{named}
\bibliography{bibliography}

\begin{thebibliography}{}

\bibitem[\protect\citeauthoryear{Bower \bgroup \em et al.\egroup
  }{2017}]{Bower17a}
Amanda Bower, Sarah~N. Kitchen, Laura Niss, Martin~J. Strauss, Alex Vargo, and
  Suresh Venkatasubramanian.
\newblock Fair pipelines.
\newblock In {\em Proceedings of the 4th Workshop on Fairness, Accountability,
  and Transparency in Machine Learning (FAT-ML)}, 2017.

\bibitem[\protect\citeauthoryear{Chouldechova}{2017}]{Chouldechova17a}
Alexandra Chouldechova.
\newblock Fair prediction with disparate impact: A study of bias in recidivism
  prediction instruments.
\newblock {\em Big Data}, 5(2):153--163, 2017.

\bibitem[\protect\citeauthoryear{Corbett-Davies \bgroup \em et al.\egroup
  }{2017}]{Corbett-Davies:2017}
Sam Corbett-Davies, Emma Pierson, Avi Feller, Sharad Goel, and Aziz Huq.
\newblock Algorithmic decision making and the cost of fairness.
\newblock In {\em Proceedings of the 23rd ACM SIGKDD International Conference
  on Knowledge Discovery and Data Mining (KDD)}, pages 797--806, 2017.

\bibitem[\protect\citeauthoryear{Dua and Graff}{2017}]{uci}
Dheeru Dua and Casey Graff.
\newblock {UCI} machine learning repository, 2017.

\bibitem[\protect\citeauthoryear{Dwork and Ilvento}{2019}]{Dwork19a}
Cynthia Dwork and Christina Ilvento.
\newblock Fairness under composition.
\newblock In {\em Proceedings of the 10th conference on Innovations in
  Theoretical Computer Science (ITCS)}, pages 33:1--33:20, 2019.

\bibitem[\protect\citeauthoryear{Dwork \bgroup \em et al.\egroup
  }{2012}]{Dwork11}
Cynthia Dwork, Moritz Hardt, Toniann Pitassi, Omer Reingold, and Richard Zemel.
\newblock Fairness through awareness.
\newblock In {\em Proceedings of the 3rd conference on Innovations in
  Theoretical Computer Science Conference (ITCS)}, pages 214--226, 2012.

\bibitem[\protect\citeauthoryear{Hardt \bgroup \em et al.\egroup
  }{2016}]{Hardt:2016}
Moritz Hardt, Eric Price, and Nathan Srebro.
\newblock Equality of opportunity in supervised learning.
\newblock In {\em Proceedings of the 30th International Conference on Neural
  Information Processing Systems (NIPS)}, pages 3323--3331, 2016.

\bibitem[\protect\citeauthoryear{Heidari and Krause}{2018}]{Heidari18a}
Hoda Heidari and Andreas Krause.
\newblock Preventing disparate treatment in sequential decision making.
\newblock In {\em Proceedings of the 27th International Joint Conference on
  Artificial Intelligence (IJCAI)}, pages 2248--2254, 2018.

\bibitem[\protect\citeauthoryear{Jabbari \bgroup \em et al.\egroup
  }{2017}]{Jabbari17a}
Shahin Jabbari, Matthew Joseph, Michael Kearns, Jamie Morgenstern, and Aaron
  Roth.
\newblock Fairness in reinforcement learning.
\newblock In {\em Proceedings of the 34th International Conference on Machine
  Learning (ICML)}, pages 1617--1626, 2017.

\bibitem[\protect\citeauthoryear{Joseph \bgroup \em et al.\egroup
  }{2016}]{Joseph16a}
Matthew Joseph, Michael Kearns, Jamie Morgenstern, and Aaron Roth.
\newblock Fairness in learning: Classic and contextual bandits.
\newblock In {\em Proceedings of the 30th International Conference on Neural
  Information Processing Systems (NIPS)}, pages 325--333, 2016.

\bibitem[\protect\citeauthoryear{Kilbertus \bgroup \em et al.\egroup
  }{2017}]{Kilbertus17}
Niki Kilbertus, Mateo Rojas~Carulla, Giambattista Parascandolo, Moritz Hardt,
  Dominik Janzing, and Bernhard Sch\"{o}lkopf.
\newblock Avoiding discrimination through causal reasoning.
\newblock In {\em Proceedings of the 31st International Conference on Neural
  Information Processing Systems (NIPS)}, pages 656--666, 2017.

\bibitem[\protect\citeauthoryear{Kleinberg and Raghavan}{2018}]{Kleinberg18a}
Jon Kleinberg and Manish Raghavan.
\newblock Selection problems in the presence of implicit bias.
\newblock In {\em Proceedings of the 9th conference on Innovations in
  Theoretical Computer Science (ITCS)}, pages 33:1--33:17, 2018.

\bibitem[\protect\citeauthoryear{Kleinberg \bgroup \em et al.\egroup
  }{2017}]{Kleinberg17a}
Jon Kleinberg, Sendhil Mullainathan, and Manish Raghavan.
\newblock Inherent trade-offs in the fair determination of risk scores.
\newblock In {\em Proceedings of the 8th conference on Innovations in
  Theoretical Computer Science (ITCS)}, pages 43:1--43:23, 2017.

\bibitem[\protect\citeauthoryear{Lambrecht and Tucker}{2018}]{careerAds}
Anja Lambrecht and E.~Tucker, Catherine.
\newblock Algorithmic bias? an empirical study into apparent gender-based
  discrimination in the display of stem career ads, March 2018.
\newblock Available at SSRN: \url{https://ssrn.com/abstract=2852260}.

\bibitem[\protect\citeauthoryear{Larson \bgroup \em et al.\egroup
  }{2016}]{Propublica}
Jeff Larson, Surya Mattu, Lauren Kirchner, and Julia Angwin.
\newblock {How We Analyzed the COMPAS Recidivism Algorithm}, 2016.
\newblock ProPublica,
  \url{https://www.propublica.org/article/how-we-analyzed-the-compas-recidivism-algorithm}.

\bibitem[\protect\citeauthoryear{Lipton \bgroup \em et al.\egroup
  }{2018}]{Lipton18a}
Zachary Lipton, Julian McAuley, and Alexandra Chouldechova.
\newblock Does mitigating ml's impact disparity require treatment disparity?
\newblock In {\em Proceedings of the 32nd International Conference on Neural
  Information Processing Systems (NIPS)}, pages 8125--8135, 2018.

\bibitem[\protect\citeauthoryear{Pedreshi \bgroup \em et al.\egroup
  }{2008}]{Pedreshi08a}
Dino Pedreshi, Salvatore Ruggieri, and Franco Turini.
\newblock Discrimination-aware data mining.
\newblock In {\em Proceedings of the 14th ACM SIGKDD International Conference
  on Knowledge Discovery and Data Mining (KDD)}, pages 560--568, 2008.

\bibitem[\protect\citeauthoryear{Perry \bgroup \em et al.\egroup
  }{2013}]{Policing}
Walter~L. Perry, Brian McInnis, Carter~C. Price, Susan~C. Smith, and John~S.
  Hollywood.
\newblock {\em Predictive Policing: The Role of Crime Forecasting in Law
  Enforcement Operations}.
\newblock Rand Corporation, 2013.

\bibitem[\protect\citeauthoryear{Rohatgi and Saleh}{2015}]{rohatgi}
Vijay~K. Rohatgi and A.K.M.E. Saleh.
\newblock {\em An Introduction to Probability and Statistics}.
\newblock Wiley Series in Probability and Statistics. Wiley, 2015.

\bibitem[\protect\citeauthoryear{Schumann \bgroup \em et al.\egroup
  }{2019}]{Schumann2018a}
Candice Schumann, Samsara~N. Counts, Jeffrey~S. Foster, and John~P. Dickerson.
\newblock The diverse cohort selection problem.
\newblock In {\em Proceedings of the International Conference on Autonomous
  Agents and Multi-Agent Systems (AAMAS)}, pages 601--609, 2019.

\bibitem[\protect\citeauthoryear{Senator}{2005}]{Senator05a}
Ted~E. Senator.
\newblock Multi-stage classification.
\newblock In {\em Proceedings of the Fifth IEEE International Conference on
  Data Mining (ICDM)}, pages 386--393, 2005.

\bibitem[\protect\citeauthoryear{Trapeznikov \bgroup \em et al.\egroup
  }{2012}]{Trapeznikov12a}
Kirill Trapeznikov, Venkatesh Saligrama, and David Casta\~{n}\'on.
\newblock Multi-stage classifier design.
\newblock In {\em Proceedings of the Asian Conference on Machine Learning},
  pages 459--474, 2012.

\bibitem[\protect\citeauthoryear{Valera \bgroup \em et al.\egroup
  }{2018}]{Valera18a}
Isabel Valera, Adish Singla, and Manuel Gomez~Rodriguez.
\newblock Enhancing the accuracy and fairness of human decision making.
\newblock In {\em Proceedings of the 32nd International Conference on Neural
  Information Processing Systems (NIPS)}, pages 1774--1783, 2018.

\bibitem[\protect\citeauthoryear{Zafar \bgroup \em et al.\egroup
  }{2017}]{Zafar17a}
Muhammad~Bilal Zafar, Isabel Valera, Manuel Gomez~Rodriguez, and Krishna~P.
  Gummadi.
\newblock Fairness beyond disparate treatment \& disparate impact: Learning
  classification without disparate mistreatment.
\newblock In {\em Proceedings of the 26th International Conference on World
  Wide Web (WWW)}, pages 1171--1180, 2017.

\end{thebibliography}

\ifthenelse{\boolean{full}}{
\pagebreak
\newpage
\clearpage
\begin{appendices}
\label{section: appendix}

\section{Utility Maximization via Linear Programming}

In this appendix, we formally justify that maximizing precision under budget constraints and fairness constraints can be done via linear programming. 
The following proposition shows how to do that with only budget constraints:
\begin{proposition}[Utility maximization as a linear program]\label{Prop:unfairOpt}
Let, for all $i \in \{1, \cdots, k\}$ and all $x_1 \dots x_{d_i}$,
\begin{align*}
    p^{(i|i-1)}_{x_1 \dots x_{d_i}} = 
    \begin{cases}
    \tilde p^{(i|i-1)}_{x_1 \dots x_{d_i}} / \tilde p^{(i-1|i-2)}_{x_1 \dots x_{d_{i-1}}}, &  \textrm{ if } \tilde p^{(i-1|i-2)}_{x_1 \dots x_{d_{i-1}}} \not = 0\\
    0,& \textrm{ otherwise},
    \end{cases}
\end{align*} 
where the variables $\tilde p^{(i|i-1)}_{x_1 \dots x_{d_i}}$ are solutions of the linear program
\begin{maxi}[2]
{\!\tilde p^{(i|i-1)}_{x_1 \dots x_{d_i}}}{\frac{1}{\alpha_k}\sum_{x_1 \dots x_d} p_{x_1 \dots x_d}^{y=1}\cdot p_{x_1 \dots x_d} \cdot \tilde p^{(k|k-1)}_{x_1 \dots x_{d_k}}}{}{}
\addConstraint{\!\sum_{x_1 \dots x_d} \!p_{x_1 \dots x_d} \cdot \tilde p^{(i|i-1)}_{x_1 \dots x_{d_i}}}{\leq \alpha_i,}{\; i < k}
\addConstraint{\!\sum_{x_1 \dots x_d} \!p_{x_1 \dots x_d} \cdot \tilde p^{(k|k-1)}_{x_1 \dots x_{d_k}}}{=\alpha_k}
\addConstraint{\!0 \leq \tilde p^{(1|0)}_{x_1 \dots x_{d_1}}}{\leq 1}
\addConstraint{\!0 \leq \tilde p^{(i|i-1)}_{x_1 \dots x_{d_i}}}{\leq \tilde p^{(i-1|i-2)}_{x_1 \dots x_{d_{i-1}}},}{\;1 < i \leq k}.
\label{unfairOpt}
\end{maxi}
Then the $p^{(i|i-1)}_{x_1 \dots x_{d_i}}$ are solutions of \eqref{pb: non-fair1} without any fairness constraint.
\end{proposition}
\begin{proof}

Using \eqref{eqppv}--\eqref{eq1selection}, we can rewrite problem \eqref{pb: non-fair1} as 
\begin{maxi}[2]
{p^{(i|i-1)}_{x_1 \dots x_{d_i}}}{\frac{1}{\alpha_k} \sum \limits_{x_1 \dots x_d } p_{x_1 \dots x_d}^{y=1} \cdot p_{x_1 \dots x_d} \cdot \prod_{i=1}^k p_{x_1 \dots x_{d_i}}^{(i|i-1)}}{}{}
\addConstraint{\sum_{x_1 \dots x_{d}} p_{x_1 \dots x_d} \cdot \prod_{j=1}^i p_{x_1 \dots x_{d_j}}^{(j|j-1)}}{\leq \alpha_i,\quad  i < k}{}
\addConstraint{\sum_{x_1 \dots x_{d}} p_{x_1 \dots x_d} \cdot \prod_{j=1}^k p_{x_1 \dots x_{d_j}}^{(j|j-1)}}{=\alpha_k}
\addConstraint{0 \leq p^{(i|i-1)}_{x_1 \dots x_{d_i}}}{\leq 1,\quad 1\leq  i \leq k.}{}{}
\label{pb: non-fair1_expanded}
\end{maxi}

Let us define the new variables 
\begin{equation}\label{eq.pipitilde}
\tilde p^{(i|i-1)}_{x_1 \dots x_{d_i}} = \prod_{j=1}^i p^{(j|j-1)}_{x_1 \dots x_{d_j}}.
\end{equation} 
By substitution, we get the linear program~\eqref{unfairOpt}. Hence, assuming that $\tilde p^{(i|i-1)}_{x_1 \dots x_{d_i}}$ are solutions of~\eqref{unfairOpt}, any $p^{(i|i-1)}_{x_1 \dots x_{d_i}}$ such that \eqref{eq.pipitilde} is satisfied (as is the case for the $p^{(i|i-1)}_{x_1 \dots x_{d_i}}$ defined in the proposition) is a solution of \eqref{pb: non-fair1}.
\end{proof}

In the following lemma, we then show that fairness constraints can also be written as linear homogeneous equations in terms of the transformed variables $\mathbf{p} = (\tilde p^{(i|i-1)}_{x_1 \dots x_{d_i}})^T$.
\begin{lemma}[Linearity of fairness constraints]
\label{proposition: owf linearity}
For both local and global fairness, and for both EO and DP, there exists a matrix $\mathbf{F}$ such that the fairness constraint can be expressed as 
\begin{equation*}
\mathbf{F} \mathbf{p} =  \mathbf{0}.
\end{equation*}
\end{lemma}
\begin{proof}
We present the proof for demographic parity, the idea is the same for equal opportunity. Let us consider the fairness constraint corresponding to stage $i$, $1 \leq i \leq k$:
$$
P(\hat y_i = 1 | x_s=0) = P(\hat y_i = 1 | x_s=1) .
$$
By expanding the left side, we obtain:
\begin{small}
\begin{align*}
P(\hat y_i = 1 | x_s=0) = \frac {\sum \limits_{x_i, i \not = s}\tilde p^{(i|i-1)}_{x_1 \dots x_s=0 \dots x_{d_i}}  \cdot p_{x_1 \dots  x_s=0 \dots x_d}}{\sum \limits_{x_i, i\not = s} p_{x_1 \dots x_s=0 \dots x_d}}.%
\end{align*}%
\end{small}%
Recall that $p_{x_1 \dots  x_s=0 \dots x_d}$ is a fixed parameter and not a decision variable. 
Thus, for both local and global fairness, the fairness constraint (equality of the probabilities for $x_s=0$ and $x_s=1$) can be represented in the form $\mathbf{F} \mathbf{p} =  \mathbf{0}$ for an appropriate $\mathbf{F}$ simply by moving all terms on the left side of the equality. 
\end{proof}

\section{Missing Proofs}

In this appendix, we provide the proofs of all results stated in the paper. We start by introducing notation that will be used throughout the proofs. 

\subsection*{Notation}

To ease the exposition, we introduce the following matrix notation for the problem \eqref{unfairOpt}.

$\bullet$~ \emph{Selection probabilities $\mathbf{p}$}. We concatenate all the selection probabilities in a single vector:
$$\mathbf{p} = (\tilde p^{(i|i-1)}_{x_1 \dots x_{d_i}})^T,$$
where $\tilde p^{(i|i-1)}_{x_1 \dots x_{d_i}}$ is the vector of selection probabilities at stage $i$ for all possible values of $x_1 \dots x_{d_i}$ (whose size depends on $i$).
    
$\bullet$~ \emph{Constraints set $C_{\bm{\alpha}_{-k}, \alpha_k}$}. We have two types of constraints in problem \eqref{unfairOpt}.
 
\begin{enumerate}
\item The constraints that correspond to selection sizes $\alpha_i$, $i=1,\dots,k$. We separate them such that $\mathbf{A} \mathbf{p} \leq \bm{\alpha}_{-k}$ corresponds to selection at first $k-1$ stages, so $$\bm{\alpha}_{-k} = (\alpha_1, \dots, \alpha_{k-1})^T.$$ The constraint $\mathbf{b}^T \mathbf{p} = \bm{\alpha}_k$ corresponds to selection at the last stage, where we require a strict equality.

\item The constraints that do not depend on selection sizes are written in a form of $\mathbf{D} \mathbf{p} \leq \bm{\delta}$ for an appropriate \textbf{D}, where $\bm{\delta} = (1, \dots, 1, 0, \dots, 0)^T$: 1's in $\bm{\delta}$ correspond to constraints $0 \leq \tilde p^{(1|0)}_{x_1 \dots x_{d_1}} \leq 1$ and 0's correspond to constraints $0 \leq \tilde p^{(i|i-1)}_{x_1 \dots x_{d_i}} \leq \tilde p^{(i-1|i-2)}_{x_1 \dots x_{d_{i-1}}}$.
\end{enumerate}

Thus, we write every constraint in matrix form and introduce the following compactly formed constraint set:
$$
    C_{\bm{\alpha}_{-k}, \alpha_k} = \{ \mathbf{p} \in [0, 1]^d:\;   \mathbf{A}  \mathbf{p} \leq \bm{\alpha}_{-k},\;
        \mathbf{b}^T  \mathbf{p}  = \alpha_k ,\;
         \mathbf{D}  \mathbf{p} \leq  \bm{\delta}
        \}.
$$
    
$\bullet$~ \emph{Utility function $U_{\bm{\alpha}_{-k}, \alpha_k}(\mathbf{p})$}, can be written as
$$
U_{ \bm{\alpha}_{-k}, \alpha_k}( \mathbf{p}) = \frac{1}{\alpha_k} \mathbf{c}^T \mathbf{p},
$$
where $\mathbf{c} = (p^{y=1}_{x_1\dots x_d} \cdot p_{x_1\dots x_d})^T$.

Proposition~\ref{Prop:unfairOpt} and Lemma~\ref{proposition: owf linearity} show that the problem of maximizing precision (or equivalently, other metrics, see Proposition~\ref{proposition: metric equivalence}) can be solved through a linear program when the selection sizes sizes $(\bm{\alpha}_{-k}, \alpha_k)$ are given constants. 
This shows that the utility maximization problem in general form can be written as:
\begin{align}
\label{pb.u_star}
{U^*( \bm{\alpha}_{-k}, \alpha_k)} = \max_{\mathbf{p} \in C_{ \bm{\alpha}_{-k}, \alpha_k} \cap C_{f}}{\frac{1}{\alpha_k} \mathbf{c}^T \mathbf{p}}{},
\end{align}
where
\begin{align*}
    C_{\bm{\alpha}_{-k}, \alpha_k} &= \{\mathbf{p} \in [0, 1]^d \;:\;  \mathbf{A} \mathbf{p} \leq \bm{\alpha}_{-k},\;
        \mathbf{b}^T \mathbf{p}  = \alpha_k ,\;
        \mathbf{D} \mathbf{p} \leq \bm{\delta}
        \},\\
    C_f &= \{\mathbf{p} \in [0, 1]^d \;:\;  \mathbf{F} \mathbf{p}  = \mathbf{0}\}.
\end{align*}

\subsection*{Proof of Proposition \ref{proposition: metric equivalence}}

We prove the equivalence only for accuracy (denoted $ACC$); the proof for other metrics follows the same idea. By expanding $ACC$, we obtain:
\begin{small}
\begin{align*}
    ACC &= P(\hat y_k = y) = P(\hat y_k=1, y=1) + P(\hat y_k=0, y=0) \\
    &= P(\hat y_k = 1, y=1) + \left( P(y=0) - P(\hat y_k=1, y=0) \right)\\
    &= 2 \cdot P(\hat y_k = 1, y=1) + P(y=0) - P(\hat y_k =1)\\
    & = 2 \cdot P(\hat y_k=1) \cdot P( y=1 | \hat y_k = 1)  \\
    &+ P(y=0) - P(\hat y_k =1).
\end{align*}
\end{small}
Since the terms $P(\hat y_k=1)$ and $P(y=0)$ are constant, maximization of precision is equivalent to maximization of $ACC$.

\subsection*{Proof of Proposition \ref{proposition: fairness notions relation}}

(1) We present the proof only for demographic parity, the proof for equal opportunity follows the same idea. We do the proof by induction. 
First consider a 2-stage selection algorithm $\hat y = (\hat y_1, \hat y_2)$.
By considering the following quantity:
\begin{small}
\begin{align*}
P(\hat y_2& =1 |\hat y_1=1, x_s=0) = \frac{P(\hat y_2=1, \hat y_1=1, x_s=0)}{P(\hat y_1=1, x_s=0)} \\
&= \frac{P(\hat y_2=1, \hat y_1=1, x_s=0) + \overbrace{P(\hat y_2=1, \hat y_1=0, x_s=0)}^{=0}}{P(\hat y_1=1, x_s=0)} \\
&= \frac{P(\hat y_2=1, x_s=0)}{P(\hat y_1=1|x_s=0) P(x_s=0)} = \frac{P(\hat y_2=1|x_s=0)}{P(\hat y_1=1|x_s=0)},%
\end{align*}%
\end{small}%
the fairness constraint for LF1 at the second stage is:
\begin{small}
\begin{align*}
\frac{P(\hat y_2=1|x_s=0)}{P(\hat y_1=1|x_s=0)} = \frac{P(\hat y_2=1|x_s=1)}{P(\hat y_1=1|x_s=1)}.
\end{align*}
\end{small}
Since we impose fairness at the first stage, then $P(\hat y_1=1|x_s=0)=P(\hat y_1=1|x_s=1)$, so the condition above is equivalent to
\begin{small}
\begin{align*}
P(\hat y_2=1|x_s=0) = P(\hat y_2=1|x_s=1),%
\end{align*}%
\end{small}%
that is exactly the second constraint for the LF2 notion. Thus, the statement is true for a 2 stage selection algorithm.

Second, assuming that the statement is true for $\hat y = (\hat y_1,\dots ,\hat y_i)$, $i > 2$, let us consider the $(i + 1)$-stage selection algorithm. By analogy, considering the quantity
\begin{small}
\begin{align*}
P(\hat y_{i+1}&=1 |\hat y_i=1, x_s=0) = \frac{P(\hat y_{i+1}=1, \hat y_i=1, x_s=0)}{P(\hat y_i=1, x_s=0)} \\
&= \frac{P(\hat y_{i+1}=1, x_s=0)}{P(\hat y_i=1|x_s=0) P(x_s=0)} = \frac{P(\hat y_{i+1}=1|x_s=0)}{P(\hat y_i=1|x_s=0)}%
\end{align*}%
\end{small}
we obtain that the fairness constraint for LF1 at the stage $i+1$ is
\begin{small}
\begin{align*}
\frac{P(\hat y_{i+1}=1|x_s=0)}{P(\hat y_i=1|x_s=0)} = \frac{P(\hat y_{i+1}=1|x_s=1)}{P(\hat y_i=1|x_s=1)}.%
\end{align*}%
\end{small}%
As $P(\hat y_i=1|x_s=0)=P(\hat y_i=1|x_s=1)$ by the assumption of induction, we have $P(\hat y_{i+1}=1|x_s=0) = P(\hat y_{i+1}=1|x_s=1)$.

The point (2) follows from the definitions of LF2 and GF, since the problem GF is less constrained than LF2.

\subsection*{Proof of Proposition \ref{proposition: monotonicity and concavity}}

(1) For a given $\bm{\alpha}_{-k}$, assume that the program attains its maximum $U^*(\bm{\alpha}_{-k}, \alpha_k)$ at point $\mathbf{p}$ and let $\bm{\alpha}^{\prime}_{-k} \geq \bm{\alpha}_{-k}$, where $\geq$ is meant component-wise. 

By setting $\mathbf{p}^{\prime} = \mathbf{p}$, we obtain that $\mathbf{p}^{\prime} \in C_{\bm{\alpha}^{\prime}_{-k}, \alpha_k}$ and thus:
$$U^*(\bm{\alpha}_{-k}, \alpha_k) = U_{\bm{\alpha}_{-k}, \alpha_k}(\mathbf{p}) = U_{\bm{\alpha}^{\prime}_{-k}, \alpha_k}(\mathbf{p}^{\prime})\leq U^*(\bm{\alpha}^{\prime}_{-k}, \alpha_k).$$
Let us consider a problem, when $ \bm{\alpha}_{-k} =  \bm{\alpha'}_{-k}$, it attains its maximum $U^*(\bm{\alpha'}_{-k}, \alpha_k)$ at the point $\mathbf{p'}$. Analogously, let for the second problem $\bm{\alpha}_{-k} =  \bm{\alpha''}_{-k}$, and it attains its maximum $U^*(\bm{\alpha''}_{-k}, \alpha_k)$ at the point $\mathbf{p''}$. Then for any $\lambda \in [0, 1]$ the point $\lambda \mathbf{p'} + (1 - \lambda)\mathbf{p''} \in C_{\lambda \bm{\alpha'}_{-k} + (1-\lambda) \bm{\alpha''}_{-k}, \alpha_k}$ and:
\begin{small}
\begin{align*}
    \lambda  U^*( \bm{\alpha'}, \alpha_k) &+ (1-\lambda) U^*( \bm{\alpha''}_{-k}, \alpha_k) = \lambda \frac{1}{\alpha_k}  \mathbf{c}^T \mathbf{p'} \\
    &+ (1 - \lambda) \frac{1}{\alpha_k}  \mathbf{c}^T \mathbf{p''} \\
    &= \frac{1}{\alpha_k} \mathbf{c}^T(\lambda \mathbf{p'} + (1-\lambda ) \mathbf{p''}) \\
    & =U_{\lambda \bm{\alpha'}_{-k} + (1-\lambda)  \bm{\alpha''}_{-k}, \alpha_k}(\lambda \mathbf{p'} + (1-\lambda ) \mathbf{p''}) \\
    &\leq U^*(\lambda \bm{\alpha'}_{-k} + (1-\lambda)  \bm{\alpha''}_{-k}, \alpha_k).
\end{align*}
\end{small}
(2) For a given $\alpha_k$, assume that the program attains its maximum $U^*(\bm{\alpha}_{-k}, \alpha_k)$ at point $\mathbf{p}$.
Let $\alpha_k^{\prime} = \alpha_k / \gamma$, where $\gamma \in [1, +\infty)$. Then consider $\mathbf{p}^{\prime} = \mathbf{p} / \gamma$. We have $\mathbf{p}^{\prime} \in C_{ \bm{\alpha}_{-k}, \alpha_k^{\prime}}$ and $\mathbf{p}^{\prime}\in C_f$ and:
\begin{align*}
U^*( \bm{\alpha}_{-k}, \alpha_k) &= U_{\bm{\alpha}_{-k}, \alpha_k}(\mathbf{p}) = \frac{1}{\alpha_k} \mathbf{c}^T \mathbf{p} \\
&= \frac{1}{\alpha_k / \gamma} \mathbf{c}^T \mathbf{p}/\gamma =  U_{\bm{\alpha}_{-k}, \alpha_k^{\prime}}(\mathbf{p}^{\prime}) 
 \leq U^*(\bm{\alpha}_{-k}, \alpha_k^{\prime}).%
\end{align*}%

\subsection*{Proof of Proposition \ref{proposition: polf bounds}}

Let us consider the trivial locally fair algorithm. It selects candidates randomly with probability $\alpha_1$ at the first stage and with probability $\alpha_i / \alpha_{i-1}$, $\forall 1 < i  \leq k$. The utility of such random algorithm is equal to $U_{random} (\bm{\alpha}_{-k}, \alpha_k)= P(y=1)$. It is obvious by definition that $$U_{random}(\bm{\alpha}_{-k}, \alpha_k) \leq U_{LF}^*(\bm{\alpha}_{-k}, \alpha_k) \leq  U_{GF}^*(\bm{\alpha}_{-k}, \alpha_k).$$
To obtain an upper bound of $U_{GF}^*(\bm{\alpha}_{-k}, \alpha_k)$ we suppose that all features are available for the selection, meaning that $\alpha_i = 1$, $\forall i < k$. Then $U_{GF}^*(\bm{\alpha}_{-k}, \alpha_k) \leq U_{un}^*(\alpha_1=1, \dots, \alpha_{k-1}=1, \alpha_k) \leq \min(P(y=1)/\alpha_k, 1).$ Thus,
\begin{align*}
PoLF(\bm{\alpha}_{-k}, \alpha_k) &\leq  \frac{\min(P(y=1)/\alpha_k, 1)}{P(y=1)} \\
&= \min\left(\frac{1}{P(y=1)},\;\frac{1}{\alpha_k}\right).
\end{align*}

\section{Additional experimental results}

In this appendix, we provide additional experimental results that support the claims in the paper---in particular by reproducing the curves in the paper for other datasets and other fairness metrics. 

\subsection{The Price of Local Fairness ($PoLF$)}

Figure~\ref{fig: polf-all-def} displays the CDF of $PoLF$ as in Figure \ref{fig: polf-all} but including the results for EO as well. 
\begin{figure}
\centering
\includegraphics[width=\linewidth]{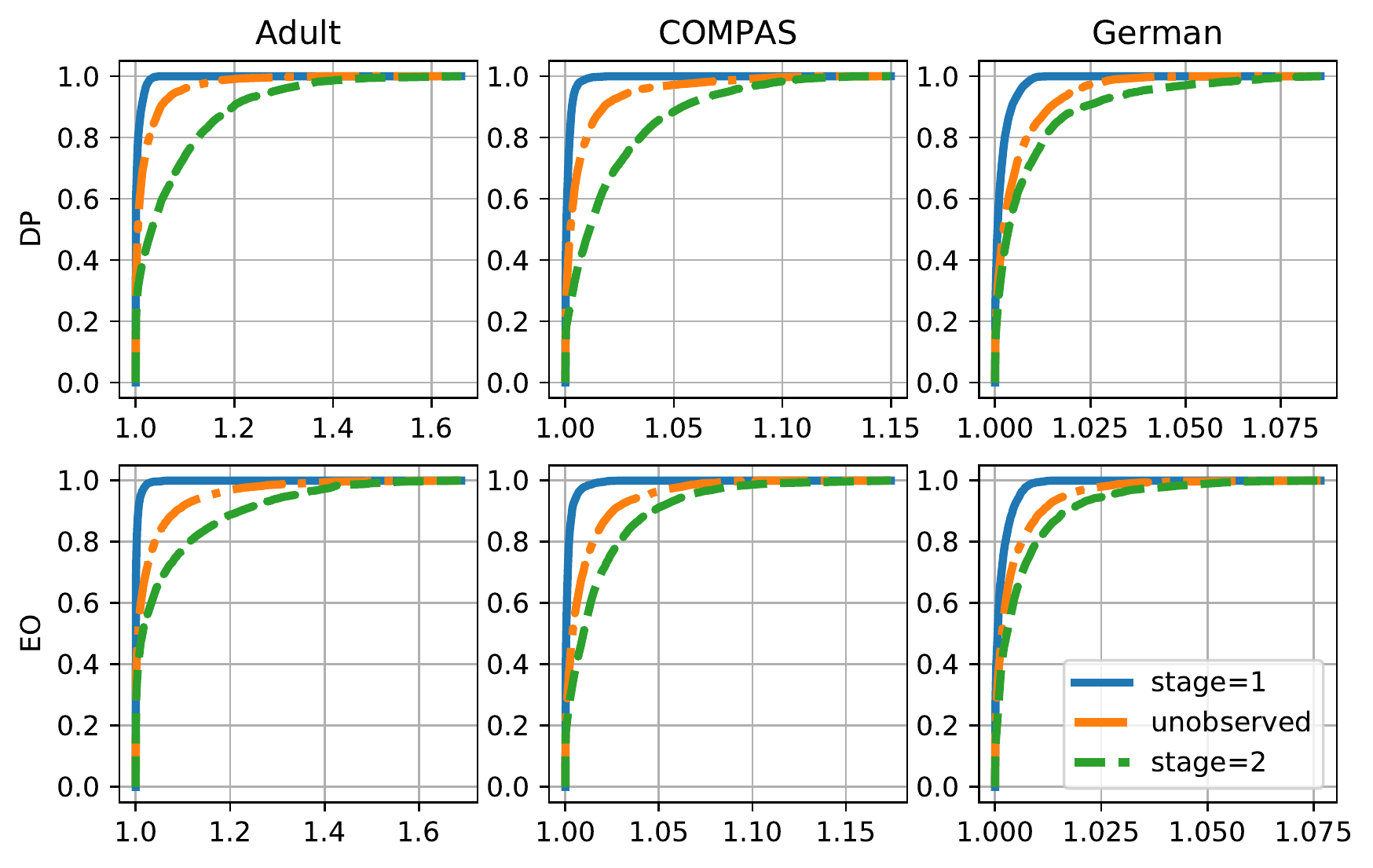}
\caption{The cumulative distribution function of $PoLF$ for $\alpha_2=0.3$ .}
\label{fig: polf-all-def}
\end{figure}


%

\subsection{The Violation of Local Fairness ($VoLF$)}

Figure~\ref{fig: volf-all} displays the CDF of $VoLF$ as in Figure~\ref{fig: volf-all} but including the results for EO as well. 
\begin{figure}[t!]
\centering
\includegraphics[width=\linewidth]{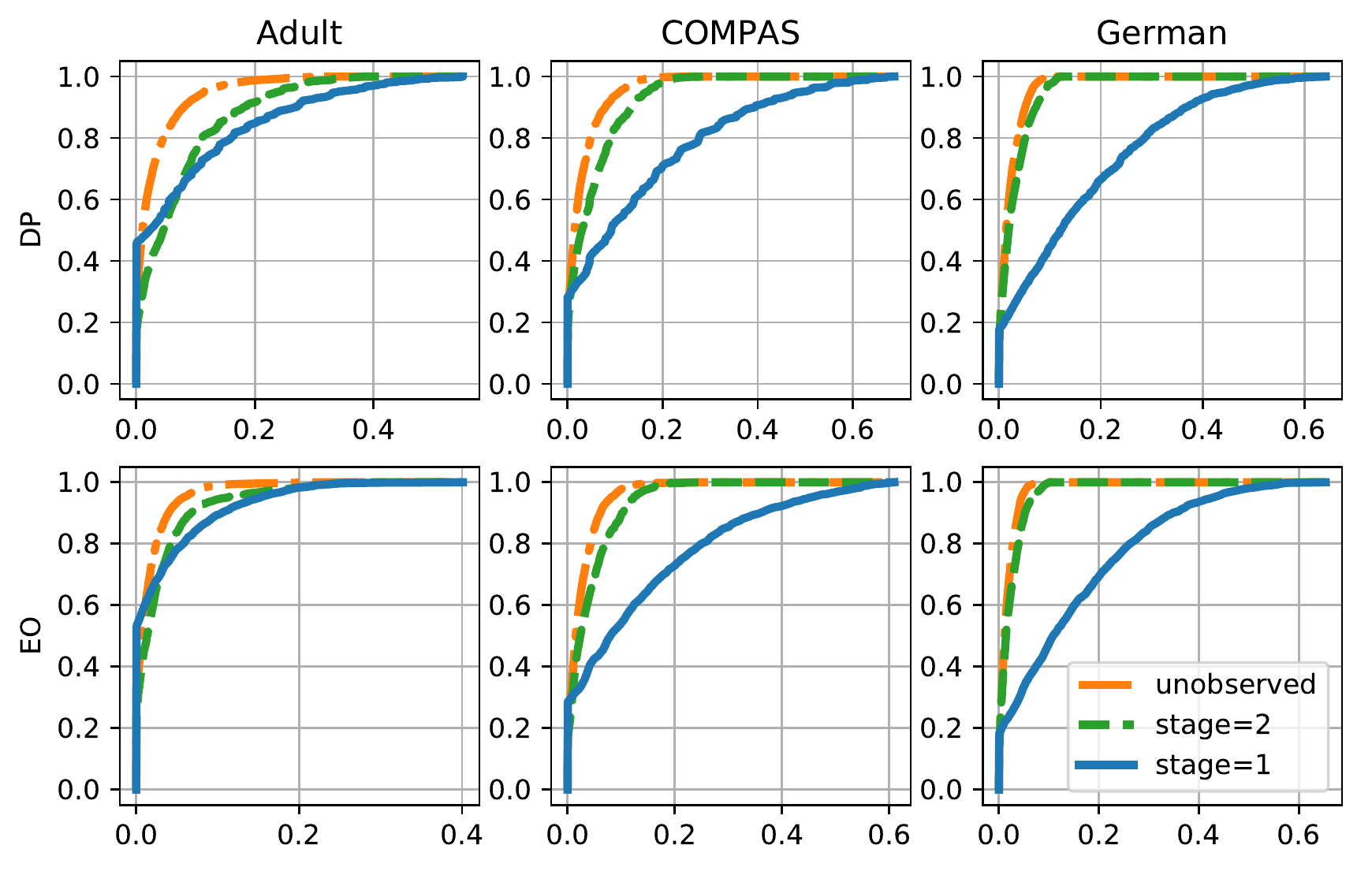}
\caption{The cumulative distribution function of $VoLF$  for $\alpha_2=0.3$.}
\label{fig: volf-all}
\end{figure}

\subsection{$VoLF$ vs $PoLF$ for Various Datasets}

Figures~\ref{fig: joint-compas} and \ref{fig: joint-german} display the joint distribution of $VoLF$ ($y$-axis) and $PoLF$ ($x$-axis) as in Figure~\ref{fig: polf-volf} but for the other two datasets: \texttt{COMPAS} and \texttt{German} respectively.
\begin{figure}[t!]
\centering
\includegraphics[width=\linewidth]{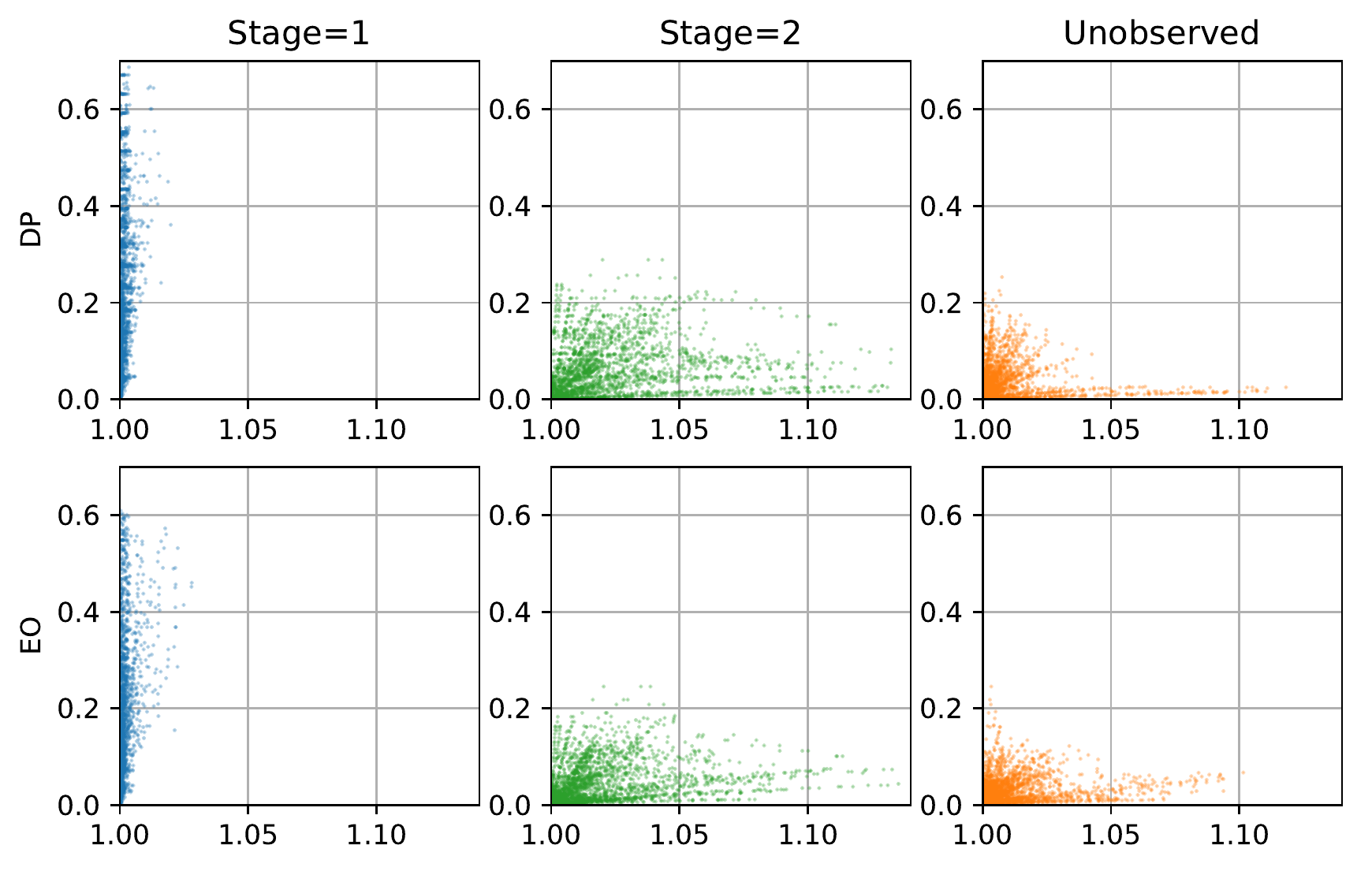}
\caption{Joint distribution of $VoLF$ and $PoLF$ for \texttt{COMPAS} dataset  and $\alpha_2=0.3$.}
\label{fig: joint-compas}
\end{figure}

\begin{figure}
\centering
\includegraphics[width=\linewidth]{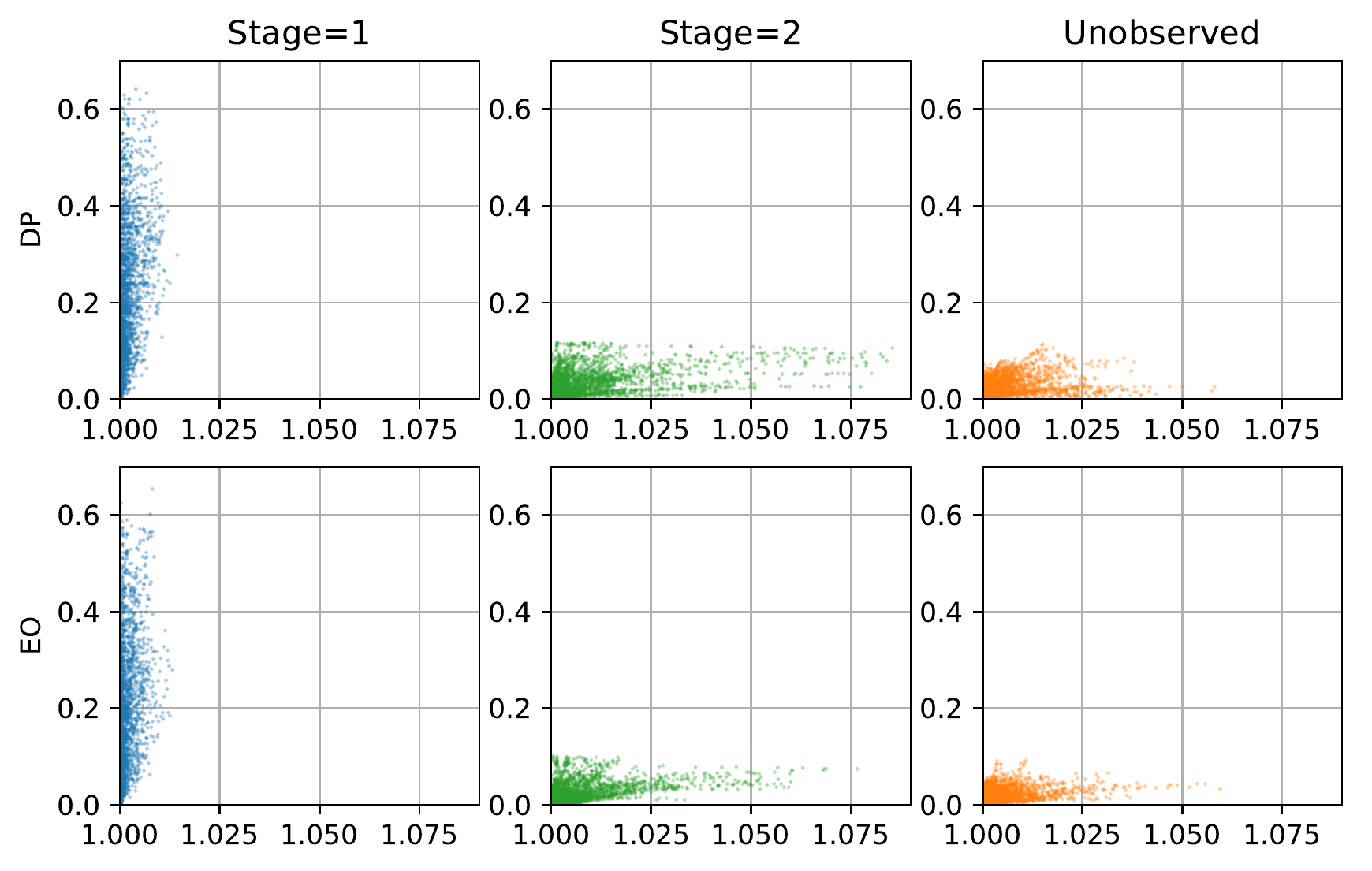}
\caption{Joint distribution of $VoLF$ and $PoLF$ for \texttt{German} dataset  and $\alpha_2=0.3$.}
\label{fig: joint-german}
\end{figure}

\subsection{$PoLF$ of $3$-stage algorithm}

In this subsection we present the results for $PoLF$ of $3$-stage selection algorithm. The procedure to calculate the $PoLF$ values is similar to the one we used for two-stage algorithm. We suppose that we observe only one feature at every stage, we suppose that the one of the rest features is a sensitive $x_s$ and consider the cases when it is observed at first, second or third stage of selection process. We calculate the value of $PoLF$ for every possible 4 feature combinations (three decision variables and one  being sensitive) out of 6 and for every discretized value of $\alpha_1$ and $\alpha_2$, such that $\alpha_3=0.3 \leq \alpha_2 \leq \alpha_1$. Figure \ref{fig: polf-3-all-def} displays the empirical CDFs of $PoLF$ for 3-stage algorithm. The observations are the same as in two-stage case: later the sensitive attribute $x_s$ is revealed, larger is the price of imposing local constraints.
\begin{figure}
\centering
\includegraphics[width=\linewidth]{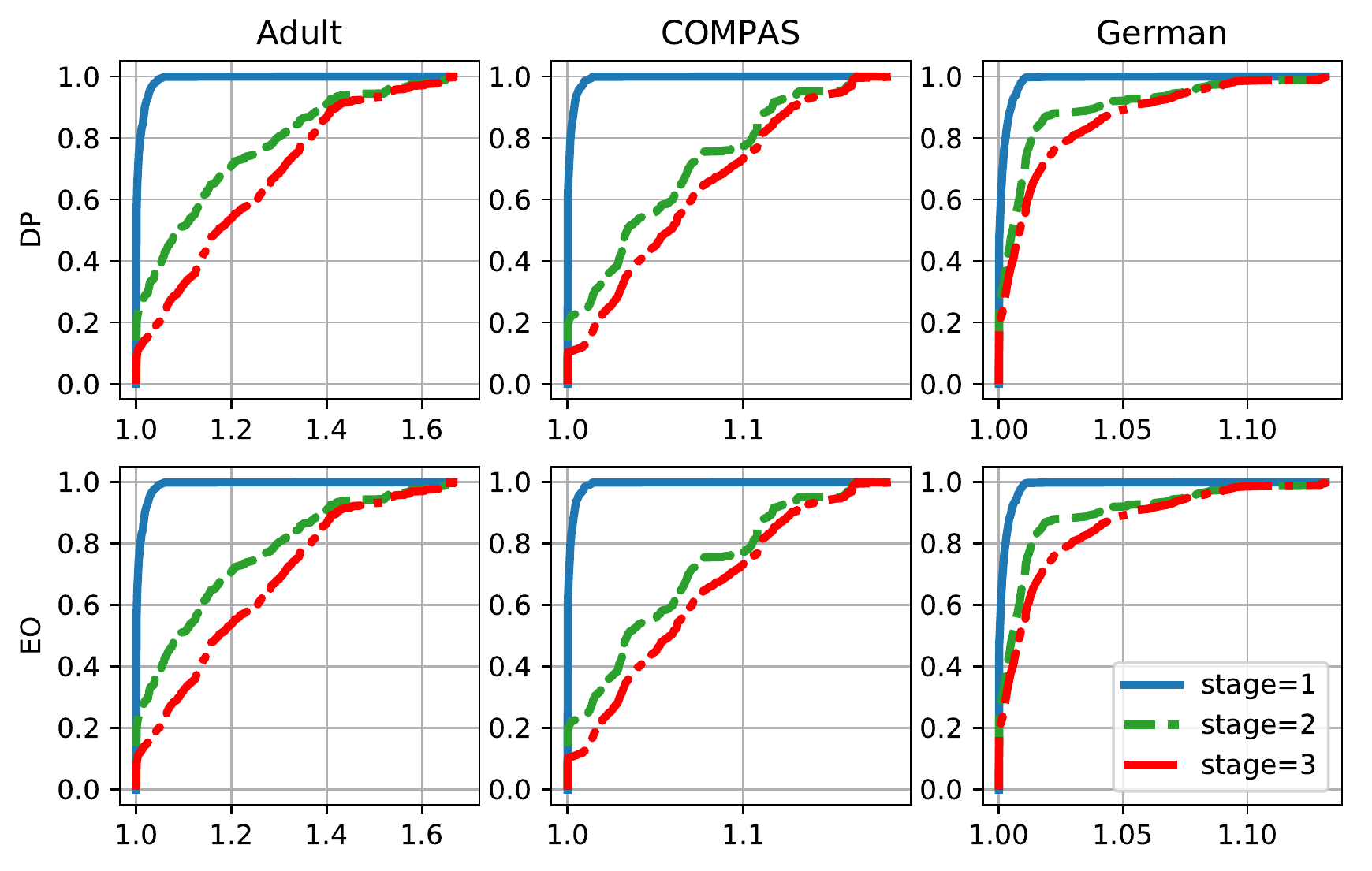}
\caption{The cumulative distribution function of $PoLF$ of $3$-stage algorithm for $\alpha_3=0.3$.}
\label{fig: polf-3-all-def}
\end{figure}

On Figure \ref{fig: polf-3-1vs2} we show the joint distribution of $PoLF$ when the $x_s$ (shown on top of each subfigure) is observed at the first stage (we call it $PoLF_1$) and $PoLF$ when the $x_s$ is observed at the second stage ($PoLF_2$) for \texttt{Adult} dataset. We observe that the value of $PoLF_1$ is sufficiently smaller than the corresponding value of $PoLF_2$.
\begin{figure}
\centering
\includegraphics[width=\linewidth]{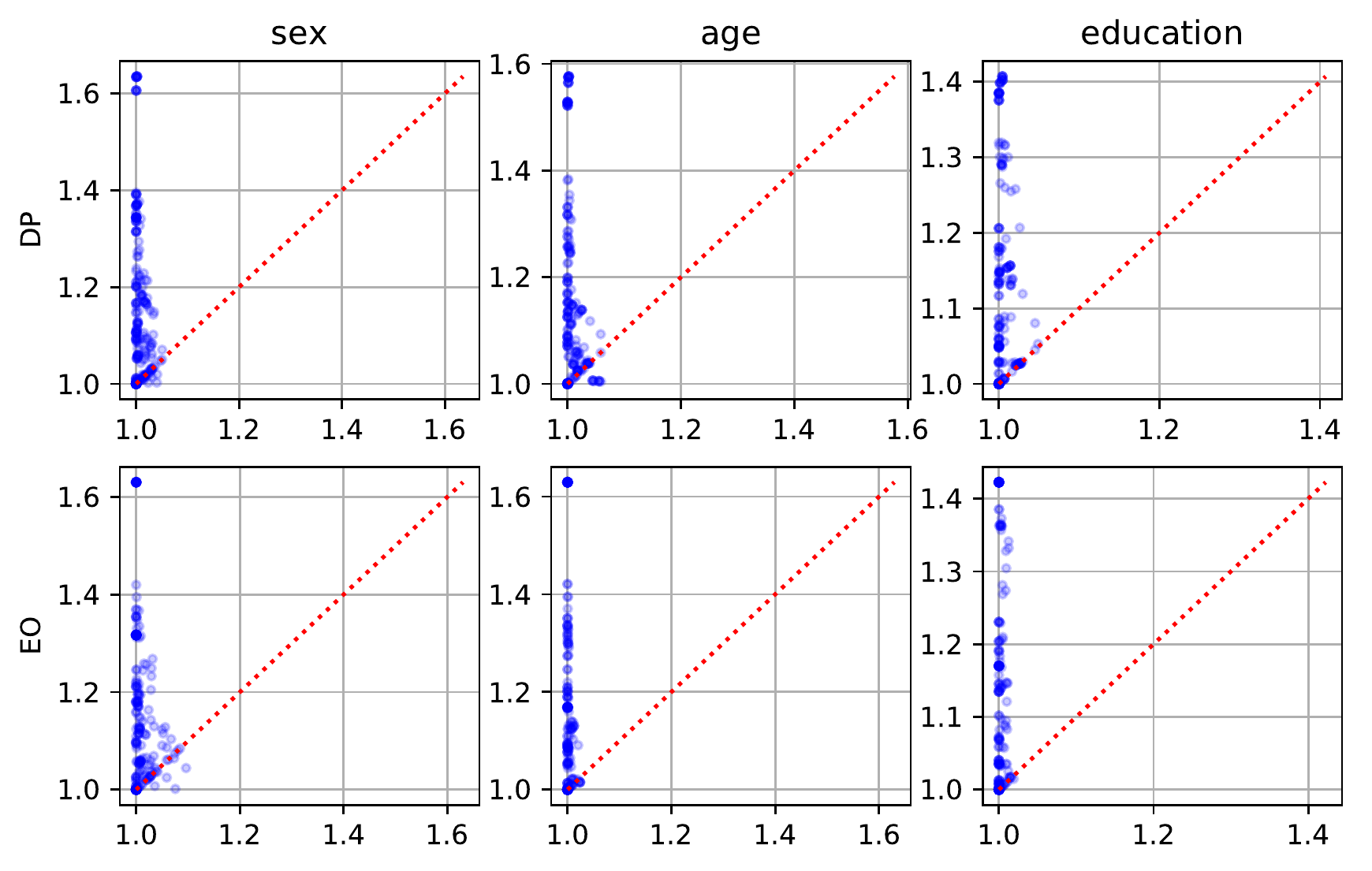}
\caption{Joint distribution of $PoLF_1$ and  $PoLF_2$ of $3$-stage algorithm for \texttt{Adult} dataset and  $\alpha_3=0.3$.}
\label{fig: polf-3-1vs2}
\end{figure}

On Figure \ref{fig: polf-3-2vs3} we show the joint distribution of $PoLF$ when the $x_s$ is observed at the second stage (we call it $PoLF_2$) and $PoLF$ when the $x_s$ is observed at the third stage ($PoLF_3$) for \texttt{Adult} dataset. We again observe that the value of $PoLF_2$ is  smaller than the corresponding value of $PoLF_3$, since the most of the points lie above the diagonal line which is marked as dashed red line.
\begin{figure}
\centering
\includegraphics[width=\linewidth]{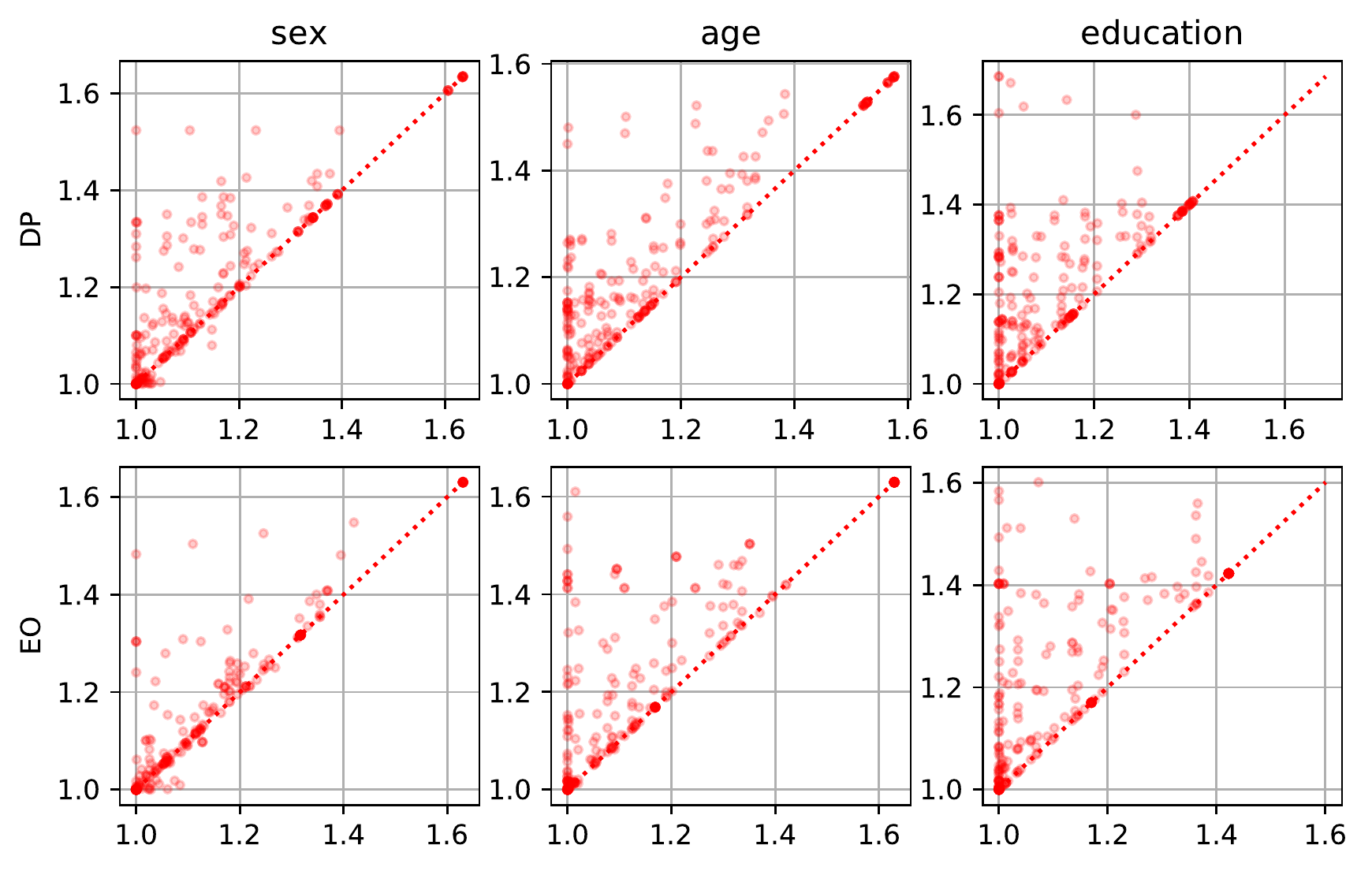}
\caption{Joint distribution of $PoLF_2$ and  $PoLF_3$ of $3$-stage algorithm  for \texttt{Adult} dataset and  $\alpha_3=0.3$.}
\label{fig: polf-3-2vs3}
\end{figure}

On Figures \ref{fig: polf-3-1vs2-compas}-\ref{fig: polf-3-2vs3-german} we display the joint distributions in the same manner as on Figures  \ref{fig: polf-3-1vs2} and  \ref{fig: polf-3-2vs3} but for \texttt{COMPAS} and \texttt{German Credit} datasets.
\begin{figure}
\centering
\includegraphics[width=\linewidth]{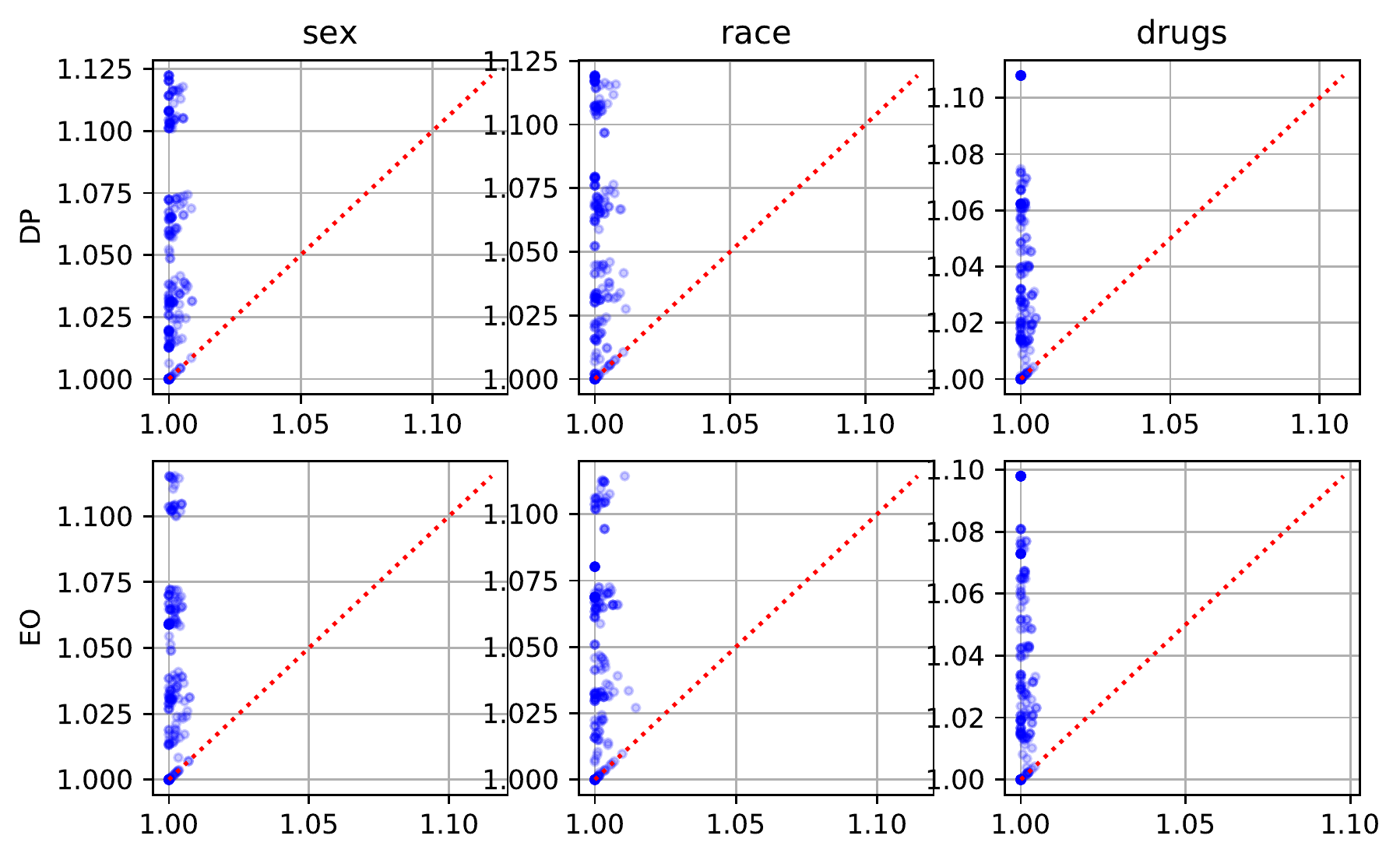}
\caption{Joint distribution of $PoLF_1$ and  $PoLF_2$ of $3$-stage algorithm for \texttt{COMPAS} dataset and  $\alpha_3=0.3$.}
\label{fig: polf-3-1vs2-compas}
\end{figure}

\begin{figure}[t!]
\centering
\includegraphics[width=\linewidth]{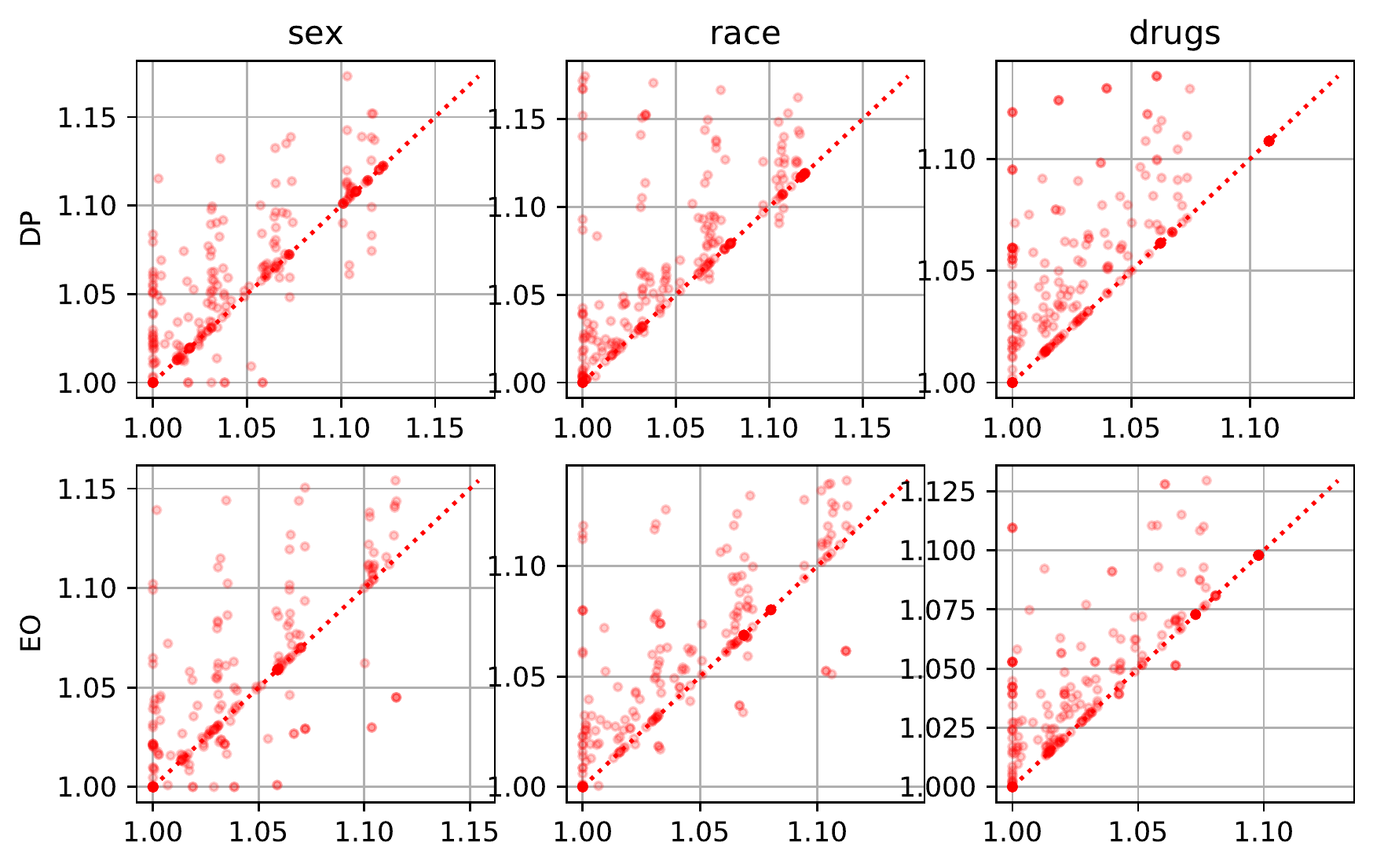}
\caption{Joint distribution of $PoLF_2$ and  $PoLF_3$ of $3$-stage algorithm  for \texttt{COMPAS} dataset and  $\alpha_3=0.3$.}
\label{fig: polf-3-2vs3-compas}
\end{figure}

\begin{figure}[t!]
\centering
\includegraphics[width=\linewidth]{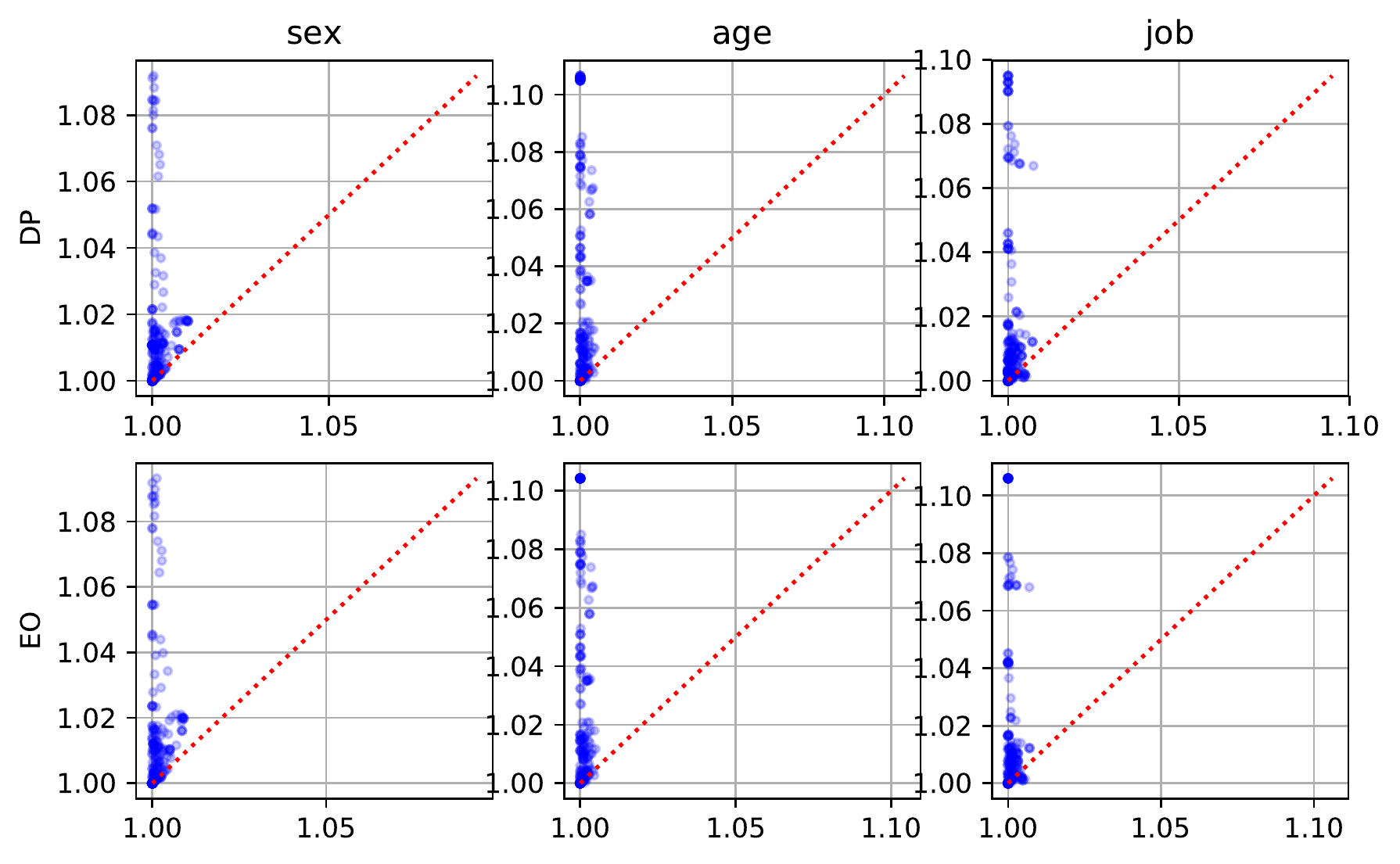}
\caption{Joint distribution of $PoLF_1$ and  $PoLF_2$ of $3$-stage algorithm for \texttt{German Credit} dataset and  $\alpha_3=0.3$.}
\label{fig: polf-3-1vs2-german}
\end{figure}

\begin{figure}
\centering
\includegraphics[width=\linewidth]{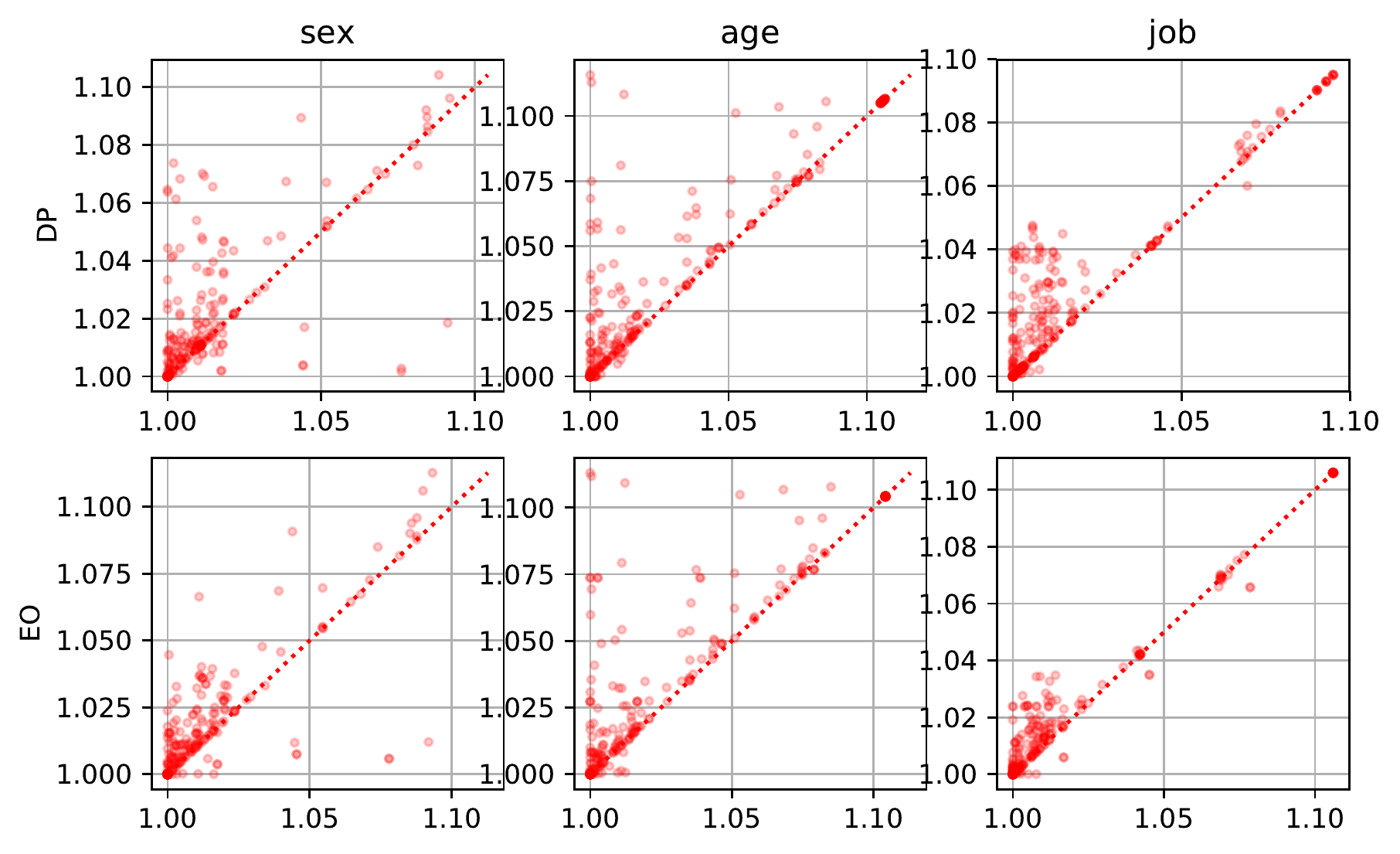}
\caption{Joint distribution of $PoLF_2$ and  $PoLF_3$ of $3$-stage algorithm  for \texttt{German Credit} dataset and  $\alpha_3=0.3$.}
\label{fig: polf-3-2vs3-german}
\end{figure}

\section{Data Description}

As mentioned in the paper, we binarize all features of datasets. In this appendix we describe this step in more detail.

\subsection{\texttt{Adult dataset}}

In the \texttt{Adult} dataset from the UCI repository \cite{uci} there are 48842 candidates, each described by 14 features. The label \textit{income} denotes if candidate gains more than 50.000 dollars annually. For all our experiments we binarize and leave only the 6 following features: \textit{sex} (is male), \textit{age} (is above 35),  \textit{native-country} (from the EU or US), \textit{education} (has Bachelor or Master degree), \textit{hours-per-week} (works more than 35 hours per week) and \textit{relationship} (is married).

\subsection{\texttt{COMPAS} Dataset}

The \texttt{COMPAS} dataset \cite{Propublica} is a dataset that is used to train the COMPAS algorithm. It contains information about prisoners, such as their name, gender, age, race, start of the sentence, end of the sentence, charge description etc. and a label $y$=\textit{recidivism}, that is $y=1$ if person is likely to reoffend and 0, otherwise.

We prepare original \texttt{COMPAS} dataset for our means by selecting statistics only for Caucasian and African-American defendants, leaving only 6 features and binarizing them. The features that we use are following: \textit{sex} (is male), \textit{young} (younger than 25), \textit{old} (older than 45),  \textit{long sentence} (sentence was longer than 30 days), \textit{drugs} (the arrest was due to selling or possessing drugs), \textit{race} (is Caucasian).

\subsection{\texttt{German} Dataset}

The \texttt{German Credit} data from \cite{uci} contains information about applicants for credit. As with other datasets, we binarize feature values. The label feature $y$=\textit{returns} shows if applicant payed for his loan, and we binarize and use 6  features: \textit{job} (is employed), \textit{housing} (owns house),  \textit{sex} (is male) \textit{savings} (greater than 500 DM),  \textit{credit history} (all credits payed back duly), \textit{age} (older than 50).

\section{Utility Maximization In Limit Of $n \rightarrow \infty$}

In this appendix we provide an intuition on multistage selection process in limit of infinitely large $n$.  In short, the optimal candidate selection problem with finite number of candidates $n$ appears to be a $k$-stage stochastic optimization problem which is difficult to solve exactly. By letting the number of candidates $n$ to be infinitely large allows us to reformulate the problem in a much simpler manner such that we are able to find an optimal selection probabilities easily.

To prove the statements in this appendix we will exploit the two following classical results from the probability theory, see \cite{rohatgi}.
\begin{lemma}[Chebyshev-Bienaym\'e inequality]
Let $X$  be a  random variable, then 
$$P\left( \left| X - EX \right| \geq \varepsilon \right) \leq \frac{Var (X)}{\varepsilon^2}.$$
\end{lemma}
\begin{lemma}[Properties of convergence in probability ]
Let $X_n$ and $Y_n$ be sequences of random variables.
\begin{enumerate}
\item If  $X_n \xrightarrow{P} X$ and $a$ is a constant, then $a X_n \xrightarrow{P} a  X$.
\item If $X_n \xrightarrow{P} X$ and $Y_n \xrightarrow{P} Y$, then $X_n + Y_n \xrightarrow{P}  X + Y$.
\item If  $X_n \xrightarrow{P} X$ , then $1 / X_n \xrightarrow{P} 1 / X$.
\item  If $X_n \xrightarrow{L} X$ and  $|X_n - Y_n| \xrightarrow{P} 0$, then $ Y_n \xrightarrow{L} X$.
\end{enumerate}
\label{lemma: convergence properties}
\end{lemma}

The following lemma gives us the limit on the proportion of selected at stage $i$ candidates as $n \rightarrow \infty$.
\begin{lemma} Let by $n^{(i)}_{x_1 \dots x_{d_i}}$ denote the number of candidates having features $x_1 \dots x_{d_i}$ that are selected at the stage $i$, then for $n \rightarrow \infty$:
$$\frac{n^{(i)}_{x_1\dots x_{d_{i}}} }{n} \xrightarrow{L} p_{x_1 \dots x_{d_i}} \prod_{j=1}^{i} p^{(j|j-1)}_{x_1 \dots x_{d_j}}.$$
\label{lemma: num selected}
\end{lemma}

Before proving the above lemma  let us define the budget $B_n(i)$ at the stage $i$ as
$B_n(i) = \frac{1}{n} \sum_{x_1 \dots x_{d_i}} n^{(i)}_{x_1 \dots x_{d_i}}.$ Then using property 2 from Lemma \ref{lemma: convergence properties} and Lemma \ref{lemma: num selected} we obtain that
 $B_n(i) \xrightarrow{L} \sum_{x_1\dots x_{d_i}} p_{x_1\dots x_{d_i}} \prod_{j=1}^i p^{(j|j-1)}_{x_1 \dots x_{d_j}}$. 
 
The precision is the proportion of good candidates among selected then using Lemma \ref{lemma: convergence properties}, the argument in Lemma \ref{lemma: num selected} and fact the the final stage selection size is fixed to $\alpha_k$, the precision converges in law to $\frac{1}{\alpha_k} \sum_{x_1\dots x_{d}} p_{x_1\dots x_{d}} p^{y=1}_{x_1\dots x_{d}} \prod_{j=1}^k p^{(j|j-1)}_{x_1 \dots x_{d_j}}$ as $n$ goes to infinity.
Hence, the equations \eqref{eqppv}--\eqref{eq1selection} hold as the number of candidates $n \rightarrow \infty$.
Let us prove Lemma \ref{lemma: num selected} by induction on stage number $i$.
\begin{proof}

\emph{Base of induction.} Before we do any selection:
\begin{align*}
n^{(0)}_{x_1\dots x_{d_1}} \sim \text{Bin}\left(n, p_{x_1 \dots x_{d_1}}\right),
\end{align*}
then using Chebyshev--Bienaym\'e inequality:
\begin{align*}
P\left(\left| \frac{n^{(0)}_{x_1\dots x_{d_1}} }{n}  - \frac{n \cdot p_{x_1 \dots x_{d_1}}}{n} \right| \geq \varepsilon  \right) &\leq \frac{n p_{x_1 \dots x_{d_1}} (1 - p_{x_1 \dots x_{d_1}}) }{\varepsilon^2  n^2 }\\
& \leq \frac{1}{\varepsilon^2 n}
\end{align*}
hence $\frac{n^{(0)}_{x_1\dots x_{d_1}} }{n} \xrightarrow{P} p_{x_1 \dots x_{d_1}}, \; n \rightarrow \infty$. 

After, when we perform the selection at the first stage:
\begin{align*}
n^{(1)}_{x_1\dots x_{d_1}} | n^{(0)}_{x_1\dots x_{d_1}} \sim \text{Bin}\left(n^{(0)}_{x_1\dots x_{d_1}},  p^{(1|0)}_{x_{1} \dots x_{d_1} }\right).
\end{align*}
\begin{align*}
P&\left(\left| \frac{n^{(1)}_{x_1\dots x_{d_1}} }{n}  - \frac{n^{(0)}_{x_1\dots x_{d_1}}}{n}p^{(1|0)}_{x_{1} \dots x_{d_1} } \right| \geq \varepsilon  \right) \leq \\
 &\leq \frac{n^{(0)}_{x_1\dots x_{d_1}} p^{(1|0)}_{x_{1} \dots x_{d_1} }  (1 - p^{(1|0)}_{x_{1} \dots x_{d_1} }  )}{\varepsilon^2 \cdot n^2 } \leq \frac{1}{\varepsilon^2 n},
\end{align*}
so using properties 1 and  4 from Lemma \ref{lemma: convergence properties}, we obtain that $\frac{n^{(1)}_{x_1\dots x_{d_1}} }{n} \xrightarrow{L} p_{x_{1} \dots x_{d_1} }  p^{(1|0)}_{x_{1} \dots x_{d_1} }$.

\emph{Induction Step.} Let us consider the stage $i+1$. By the assumption of induction:
$$\frac{n^{(i)}_{x_1 \dots x_{d_i}}}{n} \xrightarrow{L} p_{x_1\dots x_{d_i}} \prod_{j=1}^i p^{(j|j-1)}_{x_1\dots x_{d_j}}.$$
After making the selection at the stage $i$ we observe the new features $d_{i} + 1, \dots, d_{i+1}$, so
$$n^{(i)}_{x_1\dots x_{d_{i+1}}} | n^{(i)}_{x_1\dots x_{d_i}} \sim \text{Bin}\left(n^{(i)}_{x_1\dots x_{d_i}},p_{x_{d_{i}+1} \dots x_{d_{i+1}}  | x_1 \dots x_{d_i}}\right),$$
where $p_{x_{d_{i}+1} \dots x_{d_{i+1}}  | x_1 \dots x_{d_i}} := P(x_{d_{i}+1} \dots x_{d_{i+1}}  | x_1 \dots x_{d_i}).$
Then
\begin{align*}
P&\left(\left| \frac{n^{(i)}_{x_1\dots x_{d_{i+1}}} }{n}  - \frac{n^{(i)}_{x_1\dots x_{d_i}}}{n}p_{x_{d_{i}+1} \dots x_{d_{i+1}}  | x_1 \dots x_{d_i}} \right| \geq \varepsilon  \right) \leq \\
&\leq \frac{n^{(i)}_{x_1\dots x_{d_i}}  p_{x_{d_{i}+1} \dots x_{d_{i+1}}  | x_1 \dots x_{d_i}} (1- p_{x_{d_{i}+1} \dots x_{d_{i+1}}  | x_1 \dots x_{d_i}})}{\varepsilon^2 \cdot n^2 }\\
&\leq \frac{1}{\varepsilon^2 n},
\end{align*}
hence 
$ \frac{n^{(i)}_{x_1\dots x_{d_{i+1}}} }{n} \xrightarrow{L} p_{x_1 \dots x_{d_{i+1}}} \prod_{j=1}^{i} p^{(j|j-1)}_{x_1 \dots x_{d_j}}$.

When we perform the selection at the stage $i+1$:
\begin{align*}
n^{(i+1)}_{x_1\dots x_{d_{i+1}}} | n^{(i)}_{x_1\dots x_{d_{i+1}}} \sim \text{Bin}\left(n^{(i)}_{x_1 \dots x_{d_{i+1}}}, p^{(i+1|i)}_{x_1\dots x_{d_i+1}}\right)
\end{align*}
then again using Chebyshev-Bienaym\'e inequality for $\frac{n^{(i+1)}_{x_1\dots x_{d_{i+1}}} }{n}$:
\begin{align*}
P&\left(\left| \frac{n^{(i+1)}_{x_1\dots x_{d_{i+1}}} }{n}  - \frac{n^{(i)}_{x_1\dots x_{d_{i+1}}} \cdot p^{(i+1|i)}_{x_1\dots x_{d_{i+1}}}}{n} \right| \geq \varepsilon  \right) \leq\\
&\leq \frac{n^{(i)}_{x_1\dots x_{d_{i+1}}} p^{(i+1|i)}_{x_1\dots x_{d_{i+1}}}(1-p^{(i+1|i)}_{x_1\dots x_{d_{i+1}}}) }{\varepsilon^2 \cdot n^2 } \leq \frac{1}{\varepsilon^2 n},
\end{align*}
so 
$\left| \frac{n^{(i+1)}_{x_1\dots x_{d_{i+1}}} }{n} -  \frac{n^{(i)}_{x_1\dots x_{d_{i+1}}}\cdot p^{(i+1|i)}_{x_1\dots x_{d_{i+1}}} }{n} \right|  \xrightarrow{P} 0, \; n \rightarrow \infty $ and finally:
$$\frac{n^{(i+1)}_{x_1\dots x_{d_{i+1}}} }{n} \xrightarrow{L} p_{x_1 \dots x_{d_{i+1}}} \prod_{j=1}^{i+1} p^{(j|j-1)}_{x_1 \dots x_{d_j}}.$$
\end{proof}

\end{appendices}

}{}

\end{document}